\pdfoutput=1

\documentclass[11pt, a4paper, copyright, google]{google}

\usepackage[authoryear, sort&compress, round]{natbib}

\usepackage[utf8]{inputenc}
\usepackage[T1]{fontenc}
\usepackage{hyperref}
\usepackage{url}
\usepackage{booktabs}
\usepackage{nicefrac}
\usepackage{microtype}
\usepackage{graphicx}
\usepackage{multirow}
\usepackage{subcaption}
\usepackage{algorithm}
\usepackage{algorithmic}

\usepackage{amsmath}
\usepackage{amssymb}
\usepackage{mathtools}
\usepackage{amsthm}
\usepackage{thmtools}
\usepackage{thm-restate}
\usepackage{adjustbox}
\usepackage[capitalize,noabbrev]{cleveref}

\declaretheorem[numberwithin=section]{theorem}
\declaretheorem[sibling=theorem]{proposition}

\declaretheorem[sibling=theorem]{corollary}

\declaretheorem[sibling=theorem,style=remark]{remark}

\renewcommand{\theHtheorem}{\arabic{section}.\arabic{theorem}}
\renewcommand{\theHproposition}{\arabic{section}.\arabic{theorem}}
\renewcommand{\theHlemma}{\arabic{section}.\arabic{theorem}}
\renewcommand{\theHcorollary}{\arabic{section}.\arabic{theorem}}
\renewcommand{\theHdefinition}{\arabic{section}.\arabic{theorem}}
\renewcommand{\theHassumption}{\arabic{section}.\arabic{theorem}}
\renewcommand{\theHremark}{\arabic{section}.\arabic{theorem}}

\usepackage{adjustbox}
\usepackage{enumitem}
\usepackage{listings}
\lstdefinestyle{prompt}{basicstyle=\ttfamily\footnotesize, columns=fullflexible, keepspaces=true,
  breaklines=true, breakatwhitespace=true, postbreak=\mbox{$\hookrightarrow$\space},
  upquote=true, frame=single, framesep=3pt, xleftmargin=4pt, xrightmargin=4pt, aboveskip=2pt, belowskip=6pt}
\usepackage[disable,textsize=tiny]{todonotes}
\usepackage{stmaryrd}
\usepackage[safe]{tipa}
\usepackage{siunitx}
\usepackage{tikz}
\usepackage{pgfplots}
\pgfplotsset{compat=1.17}
\usetikzlibrary{decorations.pathmorphing,decorations.pathreplacing,arrows.meta}

\usepackage{bm}
\newcommand{\vecc}[1]{{\boldsymbol{\mathbf{#1}}}}

\newcommand{\vx}{{\vecc{x}}}
\newcommand{\vy}{{\vecc{y}}}

\newcommand{\vd}{{\vecc{d}}}

\newcommand{\vz}{{\vecc{z}}}

\newcommand{\vtheta}{\vecc{\theta}}

\DeclareMathOperator*{\argmax}{argmax}

\newcommand{\defn}[1]{\textbf{#1}}
\newcommand{\Real}{\mathbb{R}}

\newcommand{\E}[2][]{\mathop{\mathbb{E}}_{{#1}}[#2]}

\newcommand{\Dc}{\mathcal{D}}

\newcommand{\Yc}{\mathcal{Y}}
\newcommand{\Zc}{\mathcal{Z}}

\DeclareRobustCommand{\thinskip}{\hskip 0.16667em\relax}
\def\emdash{---}
\def\d@sh#1#2{\unskip#1\thinskip#2\thinskip\ignorespaces}
\def\Dash{\d@sh\nobreak\emdash}
\def\Ldash{\d@sh\empty{\hbox{\emdash}\nobreak}}
\def\Rdash{\d@sh\nobreak\emdash}

\crefname{ineq}{inequality}{inequalities}
\creflabelformat{ineq}{#2{\upshape(#1)}#3}

\crefname{localtheorem}{Theorem}{Theorems}
\numberwithin{localtheorem}{section}

\crefname{locallemma}{Lemma}{Lemmas}
\numberwithin{locallemma}{section}

\title{How Fast Should a Model Commit to Supervision? Training Reasoning Models on the $J_Q$ Loss Continuum}

\author[1]{Chu-Cheng Lin}
\author[1]{Eugene Ie}
\affil[1]{Google}
\correspondingauthor{kitsing@google.com, eugeneie@google.com}
\uselogo{}

\begin{abstract}
The standard recipe for training reasoning models \Dash SFT on rationale annotations followed by RLVR \Dash relies on rationale annotations being available and clean. When no annotated rationales exist, SFT cannot start at all, and when the model is too unaligned, RLVR alone stalls at cold start; when annotations are weak, SFT memorizes the errors. Using the Tsallis $q$-logarithm, we define a loss family $J_Q$ that interpolates between RLVR (at $q{=}0$, the \textit{exploitation pole}) and the log-marginal-likelihood over latent trajectories (at $q{=}1$, the \textit{density-estimation pole}). The standard pipeline corresponds to a stepwise $q{=}1 \to q{=}0$ schedule, and intermediate $q$ enables training in regimes where the two-step pipeline cannot apply. All members share the same gradient direction, differing only by a per-instance amplification $P_\vtheta^{-q}$ that reweights each instance independently of the learning rate. Under gradient flow analysis, we show that the exploitation pole requires $\Omega(\nicefrac{1}{p_0})$ time to escape cold start but is robust to label noise, while the density-estimation pole escapes in $\Theta\big(\log(\nicefrac{1}{p_0})\big)$ but memorizes label noise. This separation explains why SFT ($q{=}1$) precedes RLVR ($q{=}0$) in the standard pipeline. We further derive two Monte Carlo estimators that directly optimize fixed-$q$ on the $J_Q$ continuum, without annotated rationales or external verifiers: Gradient-Amplified RL (GARL) and Posterior-Attenuated Fine-Tuning (PAFT), with shared bias $O\big(\nicefrac{q}{M P_\vtheta^q}\big)$ but different variance and stability properties. On FinQA, HotPotQA, and MuSiQue, GARL at sufficiently high $q$ escapes cold start where GRPO fails entirely under Qwen 3 and Gemma 4 architectures. Without task prompts, cold-start GARL on Qwen 3 0.6B also matches or exceeds prompted GRPO with reshaped reward on majority vote and best-of-$k$ on all three benchmarks.
\end{abstract}

\begin{document}

\maketitle

\section{Introduction}

The standard recipe for adapting reasoning models is supervised fine-tuning (SFT) on annotated rationales followed by reinforcement learning from verifiable rewards (RLVR) \citep{ouyang2022training,guo2025deepseek,shao2024deepseekmath,chu2025sft}.
Yet two questions about it lack a unifying
theoretical account: why this specific ordering and why RLVR alone stalls at cold start (when initial  $P_{\vtheta}$ is near zero). For the latter, Rao--Blackwellized variants of RL algorithms \citep{zhou2026reinforcing} ensure non-zero gradients but, as we show, reduce variance without accelerating escape.

We provide such an account under exact-match supervision. Using the Tsallis $q$-logarithm \citep{tsallis1988possible}, we define a loss continuum $J_Q$ with a scalar \emph{commitment} parameter $q \in [0, 1]$ that interpolates between REINFORCE-style exploitation and $\log$-marginal-likelihood maximization. RLVR with exact-match reward optimizes the exploitation endpoint exactly (\cref{thm:rlvr-connection}); SFT on annotated rationales maximizes an evidence lower bound (ELBO; \citealp{Neal1998}) on the log-marginal likelihood with the annotations as a fixed variational distribution (\cref{thm:sft-elbo}). All members of $J_Q$ share one per-instance gradient direction, differing only by a factor $P_\vtheta^{-q}$ (\cref{fig:commitment}; formal definitions in \cref{sec:background}). This per-instance reweighting amplifies the gradient on unfamiliar (low-$P_\vtheta$) instances when $q$ is large \Dash an effect no global learning rate can replicate.\footnote{Adaptive and momentum optimizers such as Adam \citep{kingma2015adam} rescale updates per \emph{parameter} or average them over steps, not per \emph{example}; they cannot substitute for $P_\vtheta^{-q}$.}

The \defn{commitment} $q$ thus acts as a training-time counterpart of inference temperature: $q$ decides which examples the gradient pulls on (\cref{eq:q-gradient-geometry}), and $J_Q$'s minimizer is the $q$-tempered data distribution (\cref{thm:escort-minimizer}). High $q$ enables fast cold-start escape in $\Theta(\log(\nicefrac{1}{p_0}))$ time (\cref{thm:cold-start-escape}) but memorizes label errors (\cref{thm:noise-fitting}); low $q$ is noise-robust but escape slows to $\Omega(\nicefrac{1}{p_0})$ (\cref{thm:cold-start-lower}); both halves are confirmed empirically (\cref{sec:experiments}). This explains why SFT-then-RLVR succeeds: SFT on annotated rationales optimizes an ELBO of $J_1$, which carries the same inverse-probability amplification and so escapes cold start; switching to RLVR ($q{=}0$) afterward filters noisy supervision. It also suggests that an intermediate $q$ can cold-start a reasoning model under $J_Q$ directly, without SFT. Since $P_\vtheta$ is intractable, we estimate $\nabla_\vtheta J_Q$ by two Monte Carlo factorizations with different stability (\cref{sec:practical-mc-estimates}).

\begin{figure}[t]
\centering
\definecolor{commitA}{RGB}{44,160,44}
\definecolor{commitB}{RGB}{255,127,14}
\definecolor{commitC}{RGB}{31,119,180}
\definecolor{commitD}{RGB}{214,39,40}
\begin{tikzpicture}
\begin{axis}[
    width=9.5cm, height=6.4cm,
    xmode=log, ymode=log,
    xmin=1e-10, xmax=1, ymin=0.3, ymax=3e10,
    xtick={1e-10,1e-8,1e-6,1e-4,1e-2,1},
    ytick={1,1e2,1e4,1e6,1e8,1e10},
    xminorticks=false, yminorticks=false,
    grid=major, major grid style={gray!20},
    xlabel={success probability $P_\vtheta(\vy^* \mid \vx^*)$},
    ylabel={$\Vert \nabla_\vtheta \ell_q \Vert \,/\, \Vert \nabla_\vtheta \ell_0 \Vert$:\\ gradient amplification \\ over REINFORCE},
    ylabel style={align=center},
    tick label style={font=\footnotesize},
    label style={font=\small},
    legend pos=outer north east,
    legend style={font=\footnotesize, cells={anchor=west}},
    clip=false,
]
\addplot[draw=none, fill=blue!10, forget plot] coordinates {(1e-10,0.3) (1e-8,0.3) (1e-8,3e10) (1e-10,3e10)} -- cycle;
\addplot[black, thick] coordinates {(1e-10,1e10) (1,1)};
\addlegendentry{$q = 1$ ($\log$ marginal likelihood)}
\addplot[commitA, thick] coordinates {(1e-10,3.1623e7) (1,1)};
\addlegendentry{$q = 0.75$}
\addplot[commitB, thick] coordinates {(1e-10,1e5) (1,1)};
\addlegendentry{$q = 0.5$}
\addplot[commitC, thick] coordinates {(1e-10,316.23) (1,1)};
\addlegendentry{$q = 0.25$}
\addplot[commitD, thick] coordinates {(1e-10,1) (1,1)};
\addlegendentry{$q = 0$ (REINFORCE)}
\addplot[only marks, mark=*, mark size=1.6pt, black, forget plot] coordinates {(1e-9,1e9)};
\addplot[only marks, mark=*, mark size=1.6pt, commitA, forget plot] coordinates {(1e-9,5.6234e6)};
\addplot[only marks, mark=*, mark size=1.6pt, commitB, forget plot] coordinates {(1e-9,3.1623e4)};
\addplot[only marks, mark=*, mark size=1.6pt, commitC, forget plot] coordinates {(1e-9,177.83)};
\addplot[only marks, mark=*, mark size=1.6pt, commitD, forget plot] coordinates {(1e-9,1)};
\node[anchor=south west, inner ysep=2pt, font=\footnotesize, text=black] at (axis cs:1e-9,1e9) {$10^{9}$};
\node[anchor=south west, inner ysep=2pt, font=\footnotesize, text=commitA] at (axis cs:1e-9,5.6234e6) {$10^{6.75}$};
\node[anchor=south west, inner ysep=2pt, font=\footnotesize, text=commitB] at (axis cs:1e-9,3.1623e4) {$10^{4.5}$};
\node[anchor=south west, inner ysep=2pt, font=\footnotesize, text=commitC] at (axis cs:1e-9,177.83) {$10^{2.25}$};
\node[anchor=south west, inner ysep=2pt, font=\footnotesize, text=commitD] at (axis cs:1e-9,1) {$10^{0}$};
\draw[decorate, decoration={brace, amplitude=4pt}, blue!55!black, thick] ([yshift=2pt]axis cs:1e-10,3e10) -- ([yshift=2pt]axis cs:1e-8,3e10) node[midway, above=5pt, font=\footnotesize] {cold start};
\end{axis}
\end{tikzpicture}
\caption{Commitment as gradient amplification. Here the success probability $P_\vtheta$ is the exact-match probability of the supervised completion, marginal over rationales; $\ell_q = -\ln_q P_\vtheta$ is the per-example $J_Q$ loss (\cref{eq:single-loss}), with endpoints $\ell_0 = 1 - P_\vtheta$ (exploitation, the RLVR objective under exact match, whose gradient REINFORCE estimates) and $\ell_1 = -\log P_\vtheta$ (maximum likelihood). By \cref{thm:q-gradient-geometry},  $\ell_q$ amplifies each example's policy gradient by $P_\vtheta^{-q}$, most for answers the model finds unlikely: $q{=}0$ applies no amplification, and at $q{=}1$ the amplification $\nicefrac{1}{P_\vtheta}$ turns the policy gradient into the log-likelihood gradient. GARL applies this factor to a REINFORCE-style estimate (\cref{sec:garl}). Equivalently, relative to maximum likelihood (the $q{=}1$ line), $\ell_q$ attenuates the gradient by $P_\vtheta^{1-q}$. On log-log axes,
the amplifications at different success probabilities form straight lines of slope $-q$; these relative weights across examples survive
adaptive optimizers, which normalize gradient magnitudes per parameter with a scale shared across examples.
Under gradient flow from the same parameters, $P_\vtheta$ also grows $P_\vtheta^{-q}$ times faster than under $\ell_0$, which sets both the cold-start escape time (\cref{thm:cold-start-lower,thm:cold-start-escape}) and how fast a wrong label is fitted (\cref{thm:noise-fitting}). Shaded: for the base Qwen 3 0.6B model without task prompts, the 32-rollout estimate of $P_\vtheta$ is at most $10^{-8}$ for every prompt in the first training batch (64 prompts) on each benchmark. In the middle of that range ($P_\vtheta = 10^{-9}$, dots), $q{=}0.75$ amplifies the policy gradient about 5.6 million-fold and $q{=}1$ a billion-fold.}
\label{fig:commitment}
\end{figure}

\paragraph{Contributions.}
\textbf{(1) The $J_Q$ loss family} (\cref{sec:background,sec:gradient-geometry,sec:convergence-rates}). $J_Q$ interpolates between a bounded, noise-robust loss at $q{=}0$ and an unbounded, mode-covering loss at $q{=}1$. Its categorical minimizer is the escort $\theta_j^* \propto \alpha_j^{\nicefrac{1}{q}}$ (\cref{thm:escort-minimizer}); $J_Q$ also enforces a dispersion penalty across examples (\cref{thm:dispersion-penalty}). The shared $P_\vtheta^{-q}$ amplification separates escape speed: $\Omega(\nicefrac{1}{p_0})$ at $q{=}0$ vs.\ $\Theta(\log\nicefrac{1}{p_0})$ at $q{=}1$ (\cref{thm:cold-start-lower,thm:cold-start-escape}).
\textbf{(2) Two gradient estimators: GARL and PAFT} (\cref{sec:practical-mc-estimates}). The dual factorization yields \textbf{Gradient-Amplified RL} (prior sampling, amplified by $P_\vtheta^{-q}$; generalizes RB-REINFORCE \citep[$q{=}0$;][]{zhou2026reinforcing} and IWAE \citep[$q{=}1$;][]{burda2016importance}) and \textbf{Posterior-Attenuated Fine-Tuning} (posterior resampling, attenuated by $P_\vtheta^{1-q}$; generalizes the EM gradient update \citep[$q{=}1$;][]{dempster1977maximum,phan2023training}). Both have bias $O(\nicefrac{q}{M P_\vtheta^q})$; GARL has lower variance, while PAFT only downweights gradients and, with prompted inputs at our default learning rate, degrades gradually where GARL collapses (\cref{sec:warm-start-results}).
\textbf{(3) Empirical validation} (\cref{sec:experiments}). On FinQA, HotPotQA, and MuSiQue with exact-match training rewards: cold-start GARL at sufficiently high $q$ escapes (\cref{tab:cold-start,fig:gemma-escape}) where GRPO fails entirely, for Qwen 3 and Gemma 4 models, and under 20\% label corruption on FinQA $q{=}0.75$ loses under one point where $q{=}1$ loses 6--7, the predicted dual of the escape asymmetry (\cref{sec:noise-robustness}). Without task prompts, cold-start GARL ($q{=}0.75$, Qwen 3 0.6B) also matches or exceeds prompted GRPO on \texttt{m@16} and \texttt{p@16} on all three benchmarks, whether GRPO uses the matched (exact-match) reward or a reshaped one, which requires task knowledge (\cref{sec:cold-vs-warm}). \Cref{sec:warm-start-results} evaluates GARL and PAFT with task prompts.

\section{Setup and the \texorpdfstring{$J_Q$}{JQ} Loss Family}

\label{sec:background}
\label{sec:loss-landscape}

We consider supervised conditional generation with latent reasoning trajectories: an autoregressive language model $p_{\vtheta}$ with parameters $\vtheta \in \Real^d$, trained on a dataset $\Dc$ of input-output pairs $(\vx^*, \vy^*)$. Given input $\vx$, the model samples an unannotated latent rationale $\vz$ from $p_\vtheta(\cdot \mid \vx)$ then an output $\hat{\vy} \sim p_\vtheta(\cdot \mid \vx, \vz)$, inducing the marginal $p_\vtheta(\vy \mid \vx) = \sum_{\vz} p_\vtheta(\vz, \vy \mid \vx)$. The latent $\vz$ may be a chain of thought \citep{wei2022chain}, proof trace, program, etc.; we treat it as an \emph{operational} latent mediating the output distribution.

\paragraph{Success probability and endpoint losses.}
For each supervised example, the \defn{success probability} is $P_{\vtheta} \triangleq p_{\vtheta}(\vy^* \mid \vx^*)$. We define the \textbf{exploitation loss} $J_0(\vtheta) \triangleq \E[\Dc]{1 - P_{\vtheta}}$ and \textbf{density-estimation loss} $J_1(\vtheta) \triangleq \E[\Dc]{- \log P_{\vtheta}}$, both minimized at $P_\vtheta = 1$. Under exact-match supervision $R(\hat{\vy},\vy^*) = \mathbb{I}(\hat{\vy} = \vy^*)$, $J_0 = 1 - \E[\Dc]{\text{reward}}$ (\cref{thm:rlvr-connection}), so minimizing $J_0$ maximizes expected reward: the objective that REINFORCE and its modern variants \citep{williams1992simple,kool2019buy,shao2024deepseekmath} ascend with different gradient estimators. Conversely, $J_1$ is the negative log-marginal likelihood over latent rationales, defined from $(\vx^*, \vy^*)$ pairs alone; for any {annotation distribution} over rationales, supervised fine-tuning (SFT) on annotated rationales optimizes an evidence lower bound (ELBO) on that same log-marginal likelihood (\cref{thm:sft-elbo}).

\paragraph{The $J_Q$ family.}
The Tsallis $q$-logarithm \citep{tsallis1988possible}, $\ln_q(u) = \nicefrac{(u^{1-q}-1)}{(1-q)}$ for $u \in (0, 1]$ with $\ln_1(u) \triangleq \ln u$, defines the per-example loss and dataset objective
\begin{align}
    \ell_q(\vtheta; \vx^*, \vy^*)
    \triangleq - \ln_q P_{\vtheta}
    = \frac{1 - P_{\vtheta}^{1-q}}{1-q},
    \qquad
    J_Q(\vtheta, q) = \E[(\vx^*, \vy^*) \sim \Dc]{\ell_q(\vtheta; \vx^*, \vy^*)},
    \label{eq:single-loss}
\end{align}
recovering $J_Q(\vtheta, 0) = J_0$ and $J_Q(\vtheta, 1) = J_1$.\footnote{For observed labels, $\ell_q$ reduces to known losses for label-noise robustness; see \cref{sec:related}.} At $q < 1$ the per-example loss is bounded and noise-robust; at $q = 1$ it is unbounded and the model fits the training distribution exactly, including label errors. \Cref{fig:continuum} summarizes the two endpoints and the trade-off between them. Strict convexity of $-\ln_q$ for $q > 0$ gives $J_Q \geq -\ln_q(\E[\Dc]{P_\vtheta})$: $J_Q$ penalizes non-uniform success across examples (\emph{dispersion penalty}, \cref{thm:dispersion-penalty}).

\begin{figure}[t]
\centering
\begin{adjustbox}{max width=\linewidth}
\begin{tikzpicture}[>=stealth, font=\small]

\node[draw, rounded corners, fill=gray!10, inner sep=5pt, align=center] at (6, 4.5) {%
	\textbf{Gradient:} $\nabla_{\vtheta}\ell_q = P_{\vtheta}^{-q}\,\nabla_{\vtheta}\ell_0 = P_{\vtheta}^{1-q}\,\nabla_{\vtheta}\ell_1$ \;\; (\cref{thm:q-gradient-geometry})
};

\draw[->, thick, purple!70] (0.5, 3.2) -- (11.5, 3.2);
\node[above=2pt, purple!70!black, align=center] at (6, 3.2) {\footnotesize $q$: commitment to supervision};
\node[below=2pt, font=\scriptsize, anchor=north] at (3, 3.2) {{\color{red!70!black}$\longleftarrow$ resolves noise}};
\node[below=2pt, font=\scriptsize, anchor=north] at (9, 3.2) {{\color{blue!70!black}resolves ambiguity $\longrightarrow$}};

\node[draw, rounded corners, fill=gray!10, inner sep=5pt, align=center] at (6, 1.5) {%
    \textbf{Loss:} $\ell_q = -\ln_q(P_{\vtheta}) = \nicefrac{(1 - P_{\vtheta}^{1-q})}{(1-q)}$ \;\; (\cref{eq:single-loss})
};

\draw[very thick] (0,0) -- (12,0);
\node[below=4pt, font=\bfseries] at (0,0) {$q=0$};
\node[below=4pt, font=\bfseries] at (12,0) {$q=1$};
\foreach \x/\lab in {3/{0.25}, 6/{0.5}, 9/{0.75}} {
    \node[below=4pt, gray] at (\x,0) {\lab};
}

\fill[red!60] (0,0) circle (3pt); \fill[blue!60] (12,0) circle (3pt);
\foreach \x in {3, 6, 9} { \fill[gray] (\x,0) circle (1.5pt); }

\node[draw, rounded corners, anchor=north, align=left, inner sep=4pt, fill=red!5] at (0, -0.7) {%
    \footnotesize\textbf{Exploitation pole} \\[-1pt]
    \footnotesize \textit{Loss:} $\ell_0 = 1 - P_{\vtheta}$ (bounded) \\[-1pt]
    \footnotesize \quad mode-seeking minimizer \\[-1pt]
    \footnotesize \textit{Gradient:} $\nabla\ell_0 = -\nabla P_{\vtheta}$ \\[-1pt]
    \footnotesize \quad recovers REINFORCE \\[-1pt]
    \footnotesize Noise-robust; cold start $\Omega(p_0^{-1})$ {\scriptsize(\cref{thm:cold-start-lower})}};

\node[draw, rounded corners, anchor=north, align=left, inner sep=4pt, fill=blue!5] at (12, -0.7) {%
    \footnotesize\textbf{Density-estimation pole} \\[-1pt]
    \footnotesize \textit{Loss:} $\ell_1 = -\log P_{\vtheta}$ (unbounded) \\[-1pt]
    \footnotesize \quad mode-covering minimizer (proper) \\[-1pt]
    \footnotesize \textit{Gradient:} $\nabla\ell_1 = -\nabla\log P_{\vtheta}$ \\[-1pt]
    \footnotesize \quad gradient of log-marginal-likelihood \\[-1pt]
    \footnotesize Memorizes noise; cold start $\Theta(\log p_0^{-1})$ {\scriptsize(\cref{thm:cold-start-escape})}};

\end{tikzpicture}
\end{adjustbox}
\caption{The $J_Q$ loss family is a continuum between exploitation ($q=0$) and density estimation ($q=1$) losses (poles at either end of the axis below); correspondingly, commitment is the induced gradient amplification ($P_\vtheta^{-q}$; top arrow). High $q$ resolves ambiguity (fast cold-start escape) but also memorizes noise; low $q$ resolves noise (robust filtering) but cannot escape cold start. $p_0$ denotes initial success probability; convergence results assume bounded score (\cref{sec:convergence-rates}).}
\label{fig:continuum}
\end{figure}

\paragraph{$q$ reflects training-time {commitment}.} Just as inference temperature controls output spread at decoding time, $q$ controls it at training time: $\ell_q$ penalizes non-uniform $\vtheta$ more when $q$ increases. To illustrate this point, we consider $K$-category models with empirical frequencies $\alpha_k > 0$. $J_Q$'s minimizer for such models is the \defn{escort distribution} \citep{beck1993thermodynamics} of order $\nicefrac{1}{q}$:

\begin{restatable}{theorem}{EscortMinimizer}[Minimizers of $J_Q$ in the categorical model]
\label{thm:escort-minimizer}
For $q \in (0, 1]$, the unique minimizer of $J_Q(\theta, q) = \sum_k \alpha_k(-\ln_q \theta_k)$ over $\theta \in \Delta_K$ is $\theta_k^*(q) = \frac{\alpha_k^{\nicefrac{1}{q}}}{\sum_{k'} \alpha_{k'}^{\nicefrac{1}{q}}}$. For $q = 0$, any vertex $e_j$ with $j \in \argmax_k \alpha_k$ is optimal.
\end{restatable}
The escort interpolates continuously from full coverage ($q{=}1$: $\theta^* = \alpha$) to pure mode-seeking ($q \to 0$), with $q{=}1$ the unique strictly proper scoring rule in $J_Q$ (\cref{thm:properness}; proofs in \cref{sec:proofs-landscape}).

\paragraph{Gradient geometry.}
\label{sec:gradient-geometry}
All members of $J_Q$ share one per-example gradient direction, factoring through either the exploitation loss endpoint $\nabla_\vtheta \ell_0$ or the density-estimation loss endpoint $\nabla_\vtheta \ell_1$:

\begin{proposition}[Gradient geometry and dual factorization]
\label{thm:q-gradient-geometry}
For any fixed supervised example $(\vx^*, \vy^*)$ with $P_{\vtheta} > 0$ and any $q \in [0,1]$,
\begin{align}
    \nabla_{\vtheta} \ell_q(\vtheta; \vx^*, \vy^*)
    \;=\;
    \underbrace{P_{\vtheta}^{-q}}_{\text{amplify}}\,
    \nabla_{\vtheta} \ell_0(\vtheta; \vx^*, \vy^*)
    \;=\;
    \underbrace{P_{\vtheta}^{1-q}}_{\text{attenuate}}\,
    \nabla_{\vtheta} \ell_1(\vtheta; \vx^*, \vy^*).
    \label{eq:q-gradient-geometry}
\end{align}
\end{proposition}
\begin{proof}
By the chain rule and $\frac{d}{du}\ln_q(u) = u^{-q}$: $\nabla_{\vtheta}\ell_q = -P_{\vtheta}^{-q}\nabla_{\vtheta}P_{\vtheta} = P_{\vtheta}^{-q}\nabla_{\vtheta}\ell_0$. Since $\nabla_{\vtheta}\ell_0 = -\nabla_{\vtheta}P_{\vtheta} = P_{\vtheta}\,\nabla_{\vtheta}\ell_1$, the second equality follows.
\end{proof}

Relative to RL, $\ell_q$ amplifies each example's gradient by $P_\vtheta^{-q} \in [1, \infty)$, most for answers the model finds unlikely (\cref{fig:commitment}). This amplification controls both cold-start escape speed (\cref{sec:convergence-rates}) and ratio-estimator bias (\cref{sec:practical-mc-estimates}); the RL factorization motivates GARL (\cref{sec:garl}), the FT factorization motivates PAFT (\cref{sec:paft}).

\section{Commitment Dynamics under Gradient Flow}

\label{sec:convergence-rates}

Under gradient flow, escape from a cold start ($p_0 = P_{\vtheta(0)} \ll 1$) takes $\Omega(\nicefrac{1}{p_0})$ time at the exploitation pole ($q{=}0$) but only $\Theta(\log(\nicefrac{1}{p_0}))$ at the density-estimation pole ($q{=}1$). This exponential separation in $\nicefrac{1}{p_0}$ is governed by the amplification factor $P_{\vtheta}^{-q}$ and the dynamics $\dot p = p^{2-q}\|s(\vtheta)\|^2$.\footnote{We track the single-example success probability under continuous-time gradient flow. The main cold-start theoretical result is \emph{conservatively} stylized: interference between examples and discrete minibatch steps cannot beat the single-example speed limit (\cref{sec:beyond-single-example}), and \cref{sec:sigmoid-warmup} gives exact rates for a scalar sigmoid model.}

\paragraph{Dynamics of the success probability.}
\label{sec:cold-start-dynamics}
We study gradient flow $\dot{\vtheta} = -\nabla_\vtheta \ell(\vtheta)$ \citep{su2016differential}, which isolates closed-form rates from step-size effects without requiring convexity ($\dot p \geq 0$ always). For a single example with score $s(\vtheta) \triangleq \nabla_\vtheta \log P_\vtheta$, \cref{thm:q-gradient-geometry} gives
\begin{align}
	\dot{p}
	= \nabla_\vtheta P_\vtheta \cdot \dot{\vtheta}
	= P_\vtheta^{-q}\,\|\nabla_\vtheta P_\vtheta\|^2
	= p^{2-q}\,\|s(\vtheta)\|^2,
	\label{eq:general-p-dynamics}
\end{align}
where $q$'s entire effect on convergence is captured by the exponent $2 - q$ ($\|s\|^2$ is $q$-independent).

\paragraph{Why $q$ matters at cold start.} For $p_0 \triangleq p(0) \ll 1$ and approximately constant $\|s\|$, the time to reach target $\delta$ is $T \sim \int_{p_0}^{\delta} u^{-(2-q)}\,du$. The exponent $2 - q$ sets the divergence rate as $p_0 \to 0$: at $q = 0$, $\int u^{-2}\,du \sim p_0^{-1}$; at $q = 1$, $\int u^{-1}\,du \sim \log(1/p_0)$. %

\paragraph{Cold-start escape rates.}
\label{sec:cold-start-rates}
We present the separation in two results: an $\Omega(\cdot)$ bound assuming that score is upper-bounded (training with low-$q$ is \emph{provably slow}), then a matching $\Theta(\cdot)$ rate assuming that the score is also lower-bounded.

\begin{restatable}{theorem}{ColdStartLower}[Exploitation is provably slow]
	\label{thm:cold-start-lower}
	Let $\vtheta \in \Real^d$ parameterize any differentiable model. Consider gradient flow on $\ell_q(\vtheta) = -\ln_q(P_\vtheta)$, starting from $p_0 = P_{\vtheta(0)} \in (0, 1/2)$ with fixed target $\delta \in (p_0, 1/2]$. Suppose $\|s(\vtheta(t))\| \leq C \in \Real$. Then as $p_0 \to 0$:
	\begin{align*}
		T_q(p_0, \delta) &= \Omega\!\left(\frac{p_0^{-(1-q)}}{1-q}\right) \; \text{for } q \in [0,1), \\
		T_1(p_0, \delta) &= \Omega\!\left(\log\frac{1}{p_0}\right).
	\end{align*}
\end{restatable}

The proof integrates $\dot p \leq C^2 p^{2-q}$ (\cref{sec:proofs-convergence}). $\|s\| \leq C$ is a common regularity assumption (verified in closed form for the scalar sigmoid in \cref{sec:sigmoid-warmup}), so the exploitation pole has escape time $\Omega(\nicefrac{1}{p_0})$.

\begin{restatable}{theorem}{ColdStartEscape}[Tight cold-start escape rates]
	\label{thm:cold-start-escape}
	Under the same setup as \cref{thm:cold-start-lower}, suppose additionally that $\|s(\vtheta(t))\| \geq c > 0$ throughout the trajectory. Then as $p_0 \to 0$,
	\begin{align*}
		T_q(p_0, \delta) = \Theta\!\left(\tfrac{p_0^{-(1-q)}}{1-q}\right) \text{ for } q \in [0,1), \qquad T_1(p_0, \delta) = \Theta\!\left(\log \tfrac{1}{p_0}\right),
	\end{align*}
	and consequently $\nicefrac{T_q(p_0,\delta)}{T_{q'}(p_0,\delta)} \to \infty$ for any $q < q' \leq 1$.
\end{restatable}

The lower bound $\dot p \geq c^2 p^{2-q}$ gives the matching upper bound via the same integration (\cref{sec:proofs-convergence}). For a single example, all $\ell_q$ share one gradient direction (\cref{thm:q-gradient-geometry}), so their flows trace the same path at different speeds and see the same scores; along it $T_q = \int_{p_0}^{\delta} du/(u^{2-q}\|s\|^2)$, so $T_q \geq T_{q'}$ for $q < q'$ even where $\|s\| \geq c$ fails; only the $\Theta$ rates need it. Restricting the target to $\delta \leq 1/2$ avoids $p \to 1$, where softmax scores vanish.

\paragraph{Noise fitting is symmetric.}
The same machinery gives an exact dual: under the canonical sigmoid model, growing noise contamination from $\tilde p_0$ to a fixed target takes $T_q^{\mathrm{noise}}(\tilde p_0) = \Theta(\tilde p_0^{-(1-q)}/((1-q)\epsilon))$ for $q \in (0, 1)$ and $\Theta(\log(1/\tilde p_0)/\epsilon)$ at $q{=}1$ (\cref{thm:noise-fitting} in \cref{sec:noise-fitting}, diverging at $q{=}0$) \Dash matching cold-start escape's exponent in the small starting probability, with $\epsilon$ the only additional rate factor. So $P_\vtheta^{-q}$ accelerates clean and corrupted commitment by the same factor, and SFT-then-RL \citep{ouyang2022training,guo2025deepseek,chu2025sft} becomes a hard $q{=}1 \to q{=}0$ switch: SFT escapes cold start via $P_\vtheta^{-1}$ amplification; RL afterwards halts noise commitment ($T_q^{\mathrm{noise}} \to \infty$ at $q{=}0$). The reverse order gets neither; $J_Q$ replaces the hard switch with a smooth interpolation.

\section{Gradient Estimators for \texorpdfstring{$J_Q$}{JQ}}

\label{sec:practical-mc-estimates}

The marginal $P_\vtheta = \sum_{\vz \in \Zc} p_\vtheta(\vz, \vy^* \mid \vx^*)$ in $\nabla_\vtheta \ell_q$ is intractable, so we estimate the gradient by Monte Carlo. The dual factorization (\cref{thm:q-gradient-geometry}) yields two natural estimators:
\begin{itemize}
    \item \textbf{GARL} (\cref{sec:garl}): sample from the prior $p_{\vtheta}(\vz \mid \vx^*)$, estimate $\nabla_{\vtheta}\ell_0$ and $P_\vtheta$ from the same samples, amplify by $(\bar w_M)^{-q}$ (a plug-in estimator of the amplification factor $P_\vtheta^{-q}$).
    \item \textbf{PAFT} (\cref{sec:paft}): approximately sample from the posterior $p_{\vtheta}(\vz \mid \vx^*, \vy^*)$, estimate $\nabla_{\vtheta}\ell_1$ via teacher forcing, attenuate by $(\bar w_M)^{1-q}$ (estimating $P_\vtheta^{1-q}$).
\end{itemize}

Both are drop-in replacements for RB-RLOO \citep{zhou2026reinforcing} at the same rollout budget and need no extra forward passes: GARL adds an $O(M)$ scalar reweighting, and PAFT one categorical resample before teacher forcing on already-generated tokens.

\subsection{GARL: Gradient-Amplified RL}
\label{sec:garl}

\paragraph{A plug-in Monte Carlo estimator.}
Fix a supervised example $(\vx^*,\vy^*)$ and draw $M$ i.i.d.\ latent trajectories $\vz^{(1)},\dots,\vz^{(M)} \sim p_{\vtheta}(\cdot \mid \vx^*)$. Define the per-sample likelihood weight and gradient contribution:
\begin{align}
    w_m \triangleq p_{\vtheta}(\vy^* \mid \vx^*,\vz^{(m)}), \qquad
    g_m \triangleq -w_m\,\nabla_{\vtheta}\log p_{\vtheta}(\vz^{(m)},\vy^* \mid \vx^*),
    \label{eq:def-gm}
\end{align}
with empirical means $\bar w_M \triangleq \tfrac{1}{M}\sum_m w_m$ and $\bar g_M \triangleq \tfrac{1}{M}\sum_m g_m$. By the log-trick,
\begin{align}
    \mathbb{E}[\bar w_M] = P_{\vtheta}, \qquad
    \mathbb{E}[\bar g_M] = -\sum_{\vz} \nabla_{\vtheta} p_{\vtheta}(\vz,\vy^* \mid \vx^*) = -\nabla_{\vtheta} P_{\vtheta} = \nabla_{\vtheta}\ell_0.
    \label{eq:gm-unbiased}
\end{align}
Plugging these into the RL factorization of \cref{thm:q-gradient-geometry} yields the plug-in estimator
\begin{align}
    \hat{\nabla}_{\vtheta}\ell_q(q,\vtheta;\vx^*,\vy^*,M)
    \triangleq
    \frac{\bar g_M}{(\bar w_M)^q}.
    \label{eq:mc-estimator}
\end{align}
The dataset-level estimator of $\nabla_\vtheta J_Q$ averages \cref{eq:mc-estimator} over a minibatch: GARL amplifies the RL gradient $\bar g_M$ by the plug-in estimate $(\bar w_M)^{-q}$ of $P_\vtheta^{-q}$. At the endpoints, GARL recovers RB-REINFORCE \citep[$q{=}0$;][]{zhou2026reinforcing} and the IWAE gradient estimator \citep[$q{=}1$;][]{burda2016importance}; see \cref{sec:endpoint-recovery}.

\paragraph{Consistency and finite-sample bias.}
\Cref{eq:mc-estimator} is a ratio estimator: it reuses the same samples in numerator and denominator, so it is biased at finite $M$ even though $\bar w_M$ and $\bar g_M$ are individually unbiased.\footnote{Assumptions 1--2 are standard regularity. Assumption~3 controls the ratio-estimator denominator at fixed $\vtheta$: for autoregressive softmax models, $w_m = \prod_{t=1}^T p_\vtheta(y^*_t \mid \cdot) \geq \epsilon_0^T$ for some $\epsilon_0 > 0$. The bound is not uniform over training, and may also shrink as $P_\vtheta \to 0$.}

\begin{restatable}{theorem}{BiasExpansion}[Consistency and bias expansion]
\label{thm:mc-bias-expansion}
Fix a supervised example $(\vx^*,\vy^*)$ and assume:
\begin{enumerate}[noitemsep]
    \item $P_{\vtheta} > 0$;
    \item $\mathbb{E}[\|g_m\|^2] < \infty$;
    \item $w_m \geq \epsilon$ a.s.\ for some $\epsilon > 0$.
\end{enumerate}
Then for any fixed $q \in [0,1]$, the estimator is consistent: $\hat{\nabla}_{\vtheta}\ell_q \xrightarrow{a.s.} \nabla_{\vtheta}\ell_q$ as $M \to \infty$. Moreover, the leading-order bias is
\begin{align}
    \mathbb{E}\!\left[
        \hat{\nabla}_{\vtheta}\ell_q
    \right]
    -
    \nabla_{\vtheta}\ell_q
    \;=\;
    \frac{q}{M P_\vtheta^{q+1}}\!\left[\tfrac{q+1}{2}\nabla_\vtheta\ell_1\,\mathbf{Var}(w_m) - \mathbf{Cov}(g_m, w_m)\right] + O(M^{-2})
    \quad \text{as } M \to \infty.
    \label{eq:bias-expansion}
\end{align}
Under additionally bounded marginal and per-trajectory scores ($\|\nabla_\vtheta\log P_\vtheta\| \leq C$, $\|\nabla_\vtheta\log p_\vtheta(\vz, \vy^* \mid \vx^*)\| \leq C'$), the bracketed term is $O(P_\vtheta)$, so the bias simplifies to $O(\nicefrac{q}{M P_\vtheta^q})$.
\end{restatable}
At $q{=}0$ the bias vanishes \emph{exactly} for all $M$: the estimator reduces to the unbiased sample mean $\bar g_M$ (\cref{eq:gm-unbiased}). The proof is a delta-method expansion of $\bar g_M / \bar w_M^q$ around $(P_\vtheta, \nabla\ell_0)$ (\cref{sec:proofs-estimators}). The $J_Q$-specific feature is the joint dependence on $q$ and $P_\vtheta$: the same $P_\vtheta^{-q}$ that enables fast escape (\cref{thm:cold-start-lower,thm:cold-start-escape}) degrades estimator quality at the same rate, predicting that intermediate $q$ outperforms both endpoints \Dash confirmed in \cref{sec:experiments}.

\paragraph{Control variate.}
\label{sec:variance-reduction}
We apply the standard leave-one-out control variate \citep{kool2019buy} to GARL's score-function term, centering the per-sample coefficient $w_m/(\bar w_M)^q$ against $(\bar w_{\neg m})^{1-q}$ where $\bar w_{\neg m} \triangleq \tfrac{1}{M-1}\sum_{j \neq m} w_j$ (full RLOO estimator and derivation in \cref{sec:rloo-derivation}). The control variate preserves the bias of \cref{thm:mc-bias-expansion} (\cref{thm:rloo-bias}). At $q = 0$ this recovers the Rao--Blackwellized RLOO of \citet{zhou2026reinforcing}; at $q = 1$ the centered weight becomes $w_m/\bar w_M - 1$, a self-normalizing baseline. Pseudocode is in \cref{alg:jq-gradient}.

\subsection{PAFT: Posterior-Attenuated Fine-Tuning}
\label{sec:paft}

GARL samples from the prior and amplifies by $P_\vtheta^{-q}$ \Dash sometimes massively. The FT factorization (\cref{eq:q-gradient-geometry}) offers an alternative: sample from the posterior $p_\vtheta(\vz \mid \vx^*, \vy^*)$ \Dash where rationales already agree with $\vy^*$ \Dash and attenuate by $P_\vtheta^{1-q} \in [0,1]$.

\paragraph{Posterior form of the gradient.}
Expanding $\nabla_\vtheta\ell_1 = -\nabla_\vtheta\log P_\vtheta$ as a posterior expectation:
\begin{align}
    \nabla_\vtheta\ell_q
    &= -P_\vtheta^{1-q} \cdot
    \E[\vz \sim p_\vtheta(\vz \mid \vx^*, \vy^*)]
        {\nabla_\vtheta \log p_\vtheta(\vz, \vy^* \mid \vx^*)}.
    \label{eq:posterior-gradient}
\end{align}
Each sample gradient is standard SFT (teacher forcing) on a \emph{semantically coherent} (input, rationale, answer) triple: the rationale is posterior-weighted toward agreement with $\vy^*$.

\paragraph{Approximate posterior sampling.}
The posterior is intractable for autoregressive models. We use importance resampling \citep[IR;][]{rubin1988using}, which reuses GARL's pool and weights: resample $K$ indices $r_1, \ldots, r_K \in \{1, \ldots, M\}$ with replacement, with $r_k$ drawn proportional to $w_{r_k}$. The PAFT estimator is
\begin{align}
    \hat{\nabla}_{\text{PAFT}}
    = -(\bar w_M)^{1-q} \cdot \frac{1}{K}\sum_{k=1}^{K}
    \nabla_\vtheta \log p_\vtheta(\vz^{(r_k)}, \vy^* \mid \vx^*).
    \label{eq:paft-estimator}
\end{align}
At $q = 1$, the attenuation vanishes ($(\bar w_M)^{1-q} = 1$) and PAFT recovers the EM gradient update \Dash the M-step gradient evaluated over E-step posterior samples \citep{dempster1977maximum,phan2023training}, i.e., SFT with self-generated rationales, the zero-gap case of \cref{thm:sft-elbo}; \cref{sec:endpoint-recovery} lists all endpoint reductions.

\paragraph{Bias and variance.}
Importance resampling preserves the gradient mean: PAFT inherits GARL's leading bias expansion (\cref{thm:paft-bias}), which under the bounded-score conditions of \cref{thm:mc-bias-expansion} simplifies to $O(\nicefrac{q}{M P_\vtheta^q})$, and has strictly higher variance by the law of total variance (\cref{thm:paft-variance}; full derivations in \cref{sec:paft-bias-proof}).

GARL's lower variance comes from keeping every rationale in the gradient, with small weights on the bad ones; PAFT resamples, so low-weight rationales rarely enter its gradient, and its weights are never negative. With prompted inputs at our default learning rate, PAFT degrades gradually on HotPotQA and MuSiQue, where GARL collapses, but on FinQA GARL is stable while PAFT loses up to 20\% of its peak validation accuracy, and at a fifth of the learning rate GARL does not collapse within two epochs. Pseudocode for PAFT is in \cref{alg:paft-gradient}, and its empirical results are in \cref{sec:warm-start-results}.

\section{Empirical Validation}

\label{sec:experiments}

We validate the theoretical predictions and the empirical effectiveness of GARL on three reasoning benchmarks \Dash FinQA \citep{chen-etal-2021-finqa}, HotPotQA \citep{yang2018hotpotqa}, and MuSiQue \citep{trivedi2021musique} \Dash using Qwen 3 and Gemma 4 models \citep{yang2025qwen3technicalreport,gemmateam2026gemma4technicalreport}.

\subsection{Experimental setup}
\label{sec:experimental-setup}

Our experiments target the no-rationale-annotation regime, where SFT cannot apply, so we do not compare against SFT-then-RLVR; we evaluate fixed-$q$ GARL and leave annealed $q$ to future work. We organize the empirical findings around three research questions: \textbf{RQ1} — does increasing $q$ help escape cold start? \textbf{RQ2} — does decreasing $q$ help resist label noise, as the theory predicts? \textbf{RQ3} — without task prompts, how does cold-start GARL compare with GRPO trained on prompted inputs? \Cref{sec:warm-start-results} evaluates GARL and PAFT with prompted inputs, the standard RLVR setting.\footnote{In preliminary experiments, GARL is performant with prompted inputs but can be unstable at our default learning rate; at a fifth of that rate it did not collapse within two epochs in any of three seeds, but matched-reward GRPO barely trains at that rate (\cref{sec:warm-start-results}). We leave application of other RL training stabilization techniques on GARL as future work.}

\paragraph{Scenarios and cold-start remedies.} \emph{Cold start} uses linearized $(\vx^*, \vy^*)$ pairs with no task description and no formatting instructions; the model must discover both how to solve the problem and how to format the answer, and initial $P_\vtheta$ is very low. We consider two standard remedies:
\begin{itemize}[leftmargin=*,topsep=2pt,itemsep=1pt]
\item \textbf{Prompting} raises the initial success probability $p_0$ (\citealp{wei2022chain, kojima2022large, guo2025deepseek}): inputs carry task descriptions and answer-formatting instructions (\cref{fig:prompts}), so the un-adapted model occasionally answers correctly and reward is not sparse.
\item \textbf{Reward shaping} densifies the reward \citep{ng1999policy}: our \emph{reshaped} reward fires whenever the last \texttt{<answer>} block equals $\vy^*$,\footnote{Our implementation for Qwen 3 experiments additionally check that the block appears after the thinking-end token $\texttt{t}_{\text{end}}$.} whatever else the response contains.
\end{itemize}
Both presuppose \emph{a priori} task knowledge: prompting needs task descriptions and formatting instructions, and shaping needs a task-specific notion of partial success. Our methods need neither remedy: GARL escapes cold start by amplifying the same exact-match objective on bare $(\vx^*, \vy^*)$ pairs, with no prompt and no shaped reward.

\paragraph{Datasets, methods, and evaluation.}
GRPO and GARL use $M = 32$ rollouts per prompt during training for Qwen 3 0.6B and Gemma 4 E4B, and $M=16$ for Qwen 3 8B. All methods use 16 samples per prompt at evaluation. GARL (\cref{alg:jq-gradient}) uses the RLOO variance reduction (\cref{eq:rloo-estimator}), with advantages normalized by $M^q$ (\cref{sec:pseudocode}). PAFT is not run at cold start, where $P_\vtheta^{1-q} \approx 0$ suppresses its gradient. We evaluate $q \in \{0, 0.25, 0.5, 0.75, 1\}$ for Qwen 3 0.6B, $q \in \{0, 0.25, 0.5, 0.75, 1\}$ for Gemma 4 E4B, and $q \in \{0, 0.75, 0.85, 1\}$ for Qwen 3 8B. All methods train toward the exact-match success probability $P_\vtheta$ (\cref{sec:background}). GARL teacher-forces the canonical completion $\texttt{t}_\text{end}$\texttt{\textbackslash n<answer>}$\vy^*$\texttt{</answer>} after each sampled rationale, where $\texttt{t}_\text{end}$ indicates the thinking-end token; GRPO's \emph{matched} reward is a binary reward that requires this completion verbatim (an exact string match, with no whitespace normalization), so its expected reward is $1 - J_0$. We also report prompted GRPO trained with both matched and reshaped reward functions; evaluation uses the reshaped reward for all methods. We report \texttt{p@1} (single-sample accuracy), \texttt{p@}$k$ (best-of-$k$, rewards coverage), and \texttt{m@}$k$ (majority vote over $k$ samples \citep{wang2023selfconsistency}). Reported test numbers are taken from the checkpoint with highest validation \texttt{m@16}; unless otherwise marked with $\pm$, numbers are single-seed. Additional experiment setup details are in \cref{sec:experimental-details}.

\subsection{RQ1: Does increasing $q$ help escape cold start?}
\label{sec:cold-start-results}

Cold start tests whether $P_\vtheta^{-q}$ determines escape from a sparse-reward regime (\cref{thm:cold-start-escape}).

\begin{table}[t]
\centering
\caption{Cold-start results for Qwen 3 0.6B and 8B \citep{yang2025qwen3technicalreport} and Gemma 4 E4B \citep{gemmateam2026gemma4technicalreport}. GRPO and GARL at low $q$ score zero; the escape threshold is $q \geq 0.75$ at 0.6B, $q \geq 0.85$ at 8B, and on Gemma 4 E4B $q \geq 0.5$ on FinQA and HotPotQA and $q \geq 0.75$ on MuSiQue. Prompted GRPO with the matched or reshaped reward is a cross-regime reference (\cref{sec:cold-vs-warm}); prompted PAFT is in \cref{tab:warm-start-maj16}. Best per model $\times$ benchmark $\times$ metric in bold; $\pm$: standard deviation over 3 seeds (2 for reshaped-reward GRPO on HotPotQA), otherwise single-seed. Qwen 3 8B reshaped-reward GRPO on FinQA ran out of GPU memory at step 177 of 192; we report its best-validation checkpoint (step 100).
}
\label{tab:cold-start}
\label{tab:cold-vs-warm}
\label{tab:finqa-cold-start}
\begin{adjustbox}{max width=\textwidth}
\begin{tabular}{@{}l ccc ccc ccc@{}}
\toprule
& \multicolumn{3}{c}{FinQA} & \multicolumn{3}{c}{HotPotQA} & \multicolumn{3}{c}{MuSiQue} \\
\cmidrule(lr){2-4} \cmidrule(lr){5-7} \cmidrule(lr){8-10}
Method & \texttt{p@1} & \texttt{p@16} & \texttt{m@16} & \texttt{p@1} & \texttt{p@16} & \texttt{m@16} & \texttt{p@1} & \texttt{p@16} & \texttt{m@16} \\
\midrule
\multicolumn{10}{l}{\textit{Qwen 3 0.6B (cold-start)}} \\
GRPO & 0 & 0 & 0 & 0 & 0 & 0 & 0 & 0 & 0 \\
GRPO (prompted, matched) & 17.6 $\pm 3.9$ & 45.8 $\pm 4.2$ & 24.8 $\pm 4.3$ & 48.5 $\pm 1.5$ & 65.7 $\pm 0.7$ & 51.6 $\pm 0.7$ & 19.8 $\pm 4.7$ & 43.1 $\pm 4.9$ & 23.4 $\pm 5.0$ \\
GRPO (prompted, reshaped) & 30.1 $\pm 0.1$ & 57.9 $\pm 0.1$ & 37.1 $\pm 0.2$ & 53.2 $\pm 0.2$ & 67.7 $\pm 1.4$ & 55.4 $\pm 0.2$ & \textbf{28.9} & 51.7 & 34.0 \\
GARL $q{=}0$ (RB-RLOO) & 0 & 0 & 0 & 0 & 0 & 0 & 0 & 0 & 0 \\
GARL $q{=}0.25$ & 0 & 0 & 0 & 0 & 0 & 0 & 0 & 0 & 0 \\
GARL $q{=}0.5$ & 0 & 0 & 0 & 0 & 0 & 0 & 0 & 0 & 0 \\
GARL $q{=}0.75$ & \textbf{30.5} $\pm 0.3$ & \textbf{61.5} $\pm 0.4$ & \textbf{38.6} $\pm 0.6$ & \textbf{53.4} $\pm 0.6$ & 74.6 $\pm 1.1$ & \textbf{57.4} $\pm 0.9$ & 27.5 $\pm 1.0$ & \textbf{58.3} $\pm 0.6$ & \textbf{35.6} $\pm 1.4$ \\
GARL $q{=}1$ & 21.9 & 58.7 & 33.5 & 48.7 & \textbf{75.5} & 56.6 & 21.6 & 58.1 & 32.5 \\
\midrule
\multicolumn{10}{l}{\textit{Qwen 3 8B (cold-start)}} \\
GRPO & 0 & 0 & 0 & 0 & 0 & 0 & 0 & 0 & 0 \\
GRPO (prompted, matched) & 18.7 & 26.2 & 19.6 & 34.9 & 50.5 & 39.6 & 26.7 & 51.9 & 31.1 \\
GRPO (prompted, reshaped) & \textbf{46.9} & 59.7 & 48.7 & 63.9 & 69.1 & 64.2 & 57.3 & 71.1 & 57.6 \\
GARL $q{=}0$ & 0 & 0 & 0 & 0 & 0 & 0 & 0 & 0 & 0 \\
GARL $q{=}0.75$ & 0 & 0 & 0 & 0 & 0 & 0 & 0 & 0 & 0 \\
GARL $q{=}0.85$ & 45.0 & 75.2 & \textbf{52.9} & \textbf{64.8} & \textbf{81.5} & \textbf{68.6} & \textbf{58.7} &  78.8 & 62.9  \\
GARL $q{=}1$ & 38.4 & \textbf{75.6} & 50.1 & 61.6 & 81.4 & 67.9 & 57.1  & \textbf{79.6} & \textbf{64.5} \\
\midrule
\multicolumn{10}{l}{\textit{Gemma 4 E4B (cold-start)}} \\
GRPO & 0 & 0 & 0 & 0 & 0 & 0 & 0 & 0 & 0 \\
GRPO (prompted, matched) & 18.8 & 25.3 & 19.8 & 61.6 & 69.7 & 62.6 & \textbf{54.5} & 70.9 & 56.3 \\
GRPO (prompted, reshaped) & 19.5 & 26.2 & 20.5 & 64.2 & 69.3 & 64.6 & 43.0 & 65.5 & 50.0 \\
GARL $q{=}0$ & 0 & 0 & 0 & 0 & 0 & 0 & 0 & 0 & 0 \\
GARL $q{=}0.25$ & 0 & 0 & 0 & 0 & 0 & 0 & 0 & 0 & 0 \\
GARL $q{=}0.5$ & 43.7 & 63.0 & 46.7 & 63.3 & 75.7 & 64.3 & 0 & 0 & 0 \\
GARL $q{=}0.75$ & \textbf{45.2} & 69.8 & \textbf{50.2} & \textbf{65.5} & 80.1 & \textbf{67.4} & 50.6 & 76.0 & 57.2 \\
GARL $q{=}1$ & 33.8 & \textbf{71.2} & 46.1 & 60.0 & \textbf{80.2} & 64.6 & 51.9 & \textbf{81.2} & \textbf{59.7} \\
\bottomrule
\end{tabular}
\end{adjustbox}
\end{table}
\paragraph{Yes, across model families, but only above a critical $q$ that depends on the model.}
GRPO, Rao--Blackwellized RLOO ($q{=}0$), and all $q \leq 0.5$ fail entirely on Qwen 3 0.6B; only $q \geq 0.75$ escapes. Rao--Blackwellization \citep{zhou2026reinforcing} reduces variance but cannot accelerate escape: at $q{=}0$ the dynamics $\dot p = p^2\|s\|^2$ have no amplification. The sharp transition at $q = 0.75$ matches \cref{thm:cold-start-lower}: the lower bound $\Omega(p_0^{-(1-q)})$ grows rapidly as $q$ decreases, so the training budget sets a critical $q$ below which escape fails. The transition is not specific to one model: it is equally sharp on Qwen 3 8B, where the threshold shifts to $q \geq 0.85$ ($q{=}0.75$ now fails), and on Gemma 4 E4B, where GRPO and GARL at $q \leq 0.25$ score exactly zero on all three benchmarks, every $q \geq 0.5$ escapes on FinQA and HotPotQA, and only $q \geq 0.75$ escapes on MuSiQue. The critical $q$ moves with the model and, on Gemma 4 E4B, with the benchmark, consistent with the fixed-budget reading of \cref{thm:cold-start-lower}, under which the initial success probability sets the threshold. In most settings the best run sits below $q{=}1$: $q{=}0.75$ on Qwen 3 0.6B, $q{=}0.85$ on Qwen 3 8B, and $q{=}0.75$ on Gemma 4 E4B FinQA and HotPotQA achieve higher \texttt{p@1} than $q{=}1$. The exception is Gemma 4 E4B MuSiQue, where $q{=}0.75$ escapes on validation only at step 250 of 312 while $q{=}1$ has escaped by step 50, and $q{=}1$ scores higher on all three metrics. The reason may be attributed to both the escape-vs-bias tradeoff (\cref{thm:mc-bias-expansion}) and the noise-fitting speed (\cref{thm:noise-fitting}): $q{=}1$'s stronger amplification escapes faster but yields higher-bias gradient estimates, and it also commits faster to any errors in the training labels (\cref{sec:noise-robustness} isolates this effect by corrupting labels).

\subsection{RQ2: Does decreasing $q$ help resist label noise?}
\label{sec:noise-robustness}

\Cref{thm:noise-fitting} is the mirror image of the escape-time separation: the same $P_\vtheta^{-q}$ amplification that lets high $q$ commit quickly to correct supervision lets it commit quickly to corrupted supervision, so at a fixed training budget lower $q$ should fit label noise less. We test this on FinQA at cold start by corrupting the supervision directly (\cref{tab:noise-robustness}): for 20\% of training examples we replace the answer with a random answer drawn from another training example (plausible, wrong) and retrain from the same cold start with GARL at $q{=}1$ and $q{=}0.75$, the two settings that escape at 0.6B.

\begin{table}[t]
\centering
\caption{Robustness to 20\% label noise on FinQA (Qwen 3 0.6B, cold-start GARL). Each setting is compared with its own clean run; $\Delta$ is the absolute drop with the relative drop in parentheses. All runs are single-seed except for GARL $q{=}0.75$ clean, where the numbers are taken from the 3-seed mean reported in \cref{tab:cold-start}.}
\label{tab:noise-robustness}
\begin{tabular}{@{}l ccc ccc@{}}
\toprule
& \multicolumn{3}{c}{GARL $q{=}1$} & \multicolumn{3}{c}{GARL $q{=}0.75$} \\
\cmidrule(lr){2-4} \cmidrule(lr){5-7}
Metric & clean & 20\% noise & $\Delta$ & clean & 20\% noise & $\Delta$ \\
\midrule
\texttt{p@1}  & 21.9 & 15.6 & $-6.3$ ($-29\%$) & 30.5 & 29.6 & $-0.9$ ($-3\%$) \\
\texttt{p@16} & 58.7 & 51.4 & $-7.3$ ($-12\%$) & 61.5 & 61.3 & $-0.2$ ($-0.3\%$) \\
\texttt{m@16} & 33.5 & 27.3 & $-6.2$ ($-19\%$) & 38.6 & 38.1 & $-0.5$ ($-1\%$) \\
\bottomrule
\end{tabular}
\end{table}

\paragraph{Yes, the predicted asymmetry appears on every metric.}
At $q{=}1$ the corruption costs 6--7 points (12--29\% relative); at $q{=}0.75$ it costs under one point (0.3--3\% relative) from identical labels, and the noised $q{=}0.75$ run still exceeds the clean $q{=}1$ run on all three metrics. Together with \cref{tab:cold-start}, this brackets the commitment axis from both ends: the ordering predicted by the theory holds for escape and for memorization, and an intermediate $q$ keeps enough of each to train in a single stage from imperfect supervision.

\subsection{RQ3: How does cold-start GARL compare with prompted GRPO?}
\label{sec:cold-vs-warm}

\paragraph{GARL matches or exceeds prompted GRPO on \texttt{m@16} and \texttt{p@16} under both matched and reshaped rewards.}
Unprompted GRPO scores zero at cold start, so we compare against prompted GRPO (\cref{tab:cold-start}). On Qwen 3 0.6B, GARL without prompts, at $q{=}0.75$ (the $q$ with the highest validation \texttt{m@16} on every benchmark), exceeds matched GRPO on every metric across all three benchmarks ($+5.8$ to $+13.8$ \texttt{m@16}). With the reshaped reward, prompted GRPO is much stronger ($37.1$, $55.4$, $34.0$ \texttt{m@16} on FinQA, HotPotQA, MuSiQue), yet cold-start GARL still leads on \texttt{m@16} ($+1.5$, $+2.0$, $+1.6$) and \texttt{p@16} ($+3.6$, $+6.9$, $+6.6$) and is level on \texttt{p@1} ($+0.4$, $+0.2$, $-1.4$); the reshaped-reward HotPotQA baseline has two seeds and the MuSiQue baseline one, and the MuSiQue \texttt{m@16} margin is about one GARL seed standard deviation. On Qwen 3 8B, matched-reward GRPO received no reward at any training step, so its row is effectively the untrained prompted model; against reshaped-reward GRPO, GARL at both escaped $q$ leads on \texttt{m@16} ($+1.4$ to $+6.9$) and \texttt{p@16} ($+7.7$ to $+15.9$) but not on \texttt{p@1} ($-8.5$ to $+1.4$). On Gemma 4 E4B, GARL at the validation-selected $q$ ($0.75$ on FinQA and HotPotQA, $1$ on MuSiQue) leads reshaped-reward GRPO on every metric ($+29.7$, $+2.8$, and $+9.7$ \texttt{m@16}) and leads matched-reward GRPO on \texttt{m@16} ($+30.4$, $+4.8$, $+3.4$) and \texttt{p@16}; it trails matched-reward GRPO only on MuSiQue \texttt{p@1} ($-2.6$). Reward shaping does not account for GRPO's cold-start failure: the un-prompted 0.6B model starts at exactly zero validation accuracy even under the reshaped reward, and GARL at $q{=}0$, which uses the same teacher-forced signal as GARL at $q{=}0.75$ without amplification, also fails there.

\section{Related Work}
\label{sec:related}

For observed labels, the same loss appears, up to reparameterization, as the $L_q$-likelihood \citep{ferrari2010maximum}, generalized cross-entropy \citep{zhang2018generalized}, and $\alpha$-loss \citep{sypherd2022tunable}, which target robustness to label noise. We apply the family to a marginal over latent rationales, where it connects RL and maximum marginal likelihood. Prior connections between RL and maximum likelihood \citep{Levine2018ReinforcementLA,norouzi2016reward,guu2017language} treat them as distinct frameworks, and latent-rationale methods sit at one endpoint: RLVR \citep{shao2024deepseekmath,guo2025deepseek} and RB-REINFORCE \citep{zhou2026reinforcing} at $q{=}0$, STaR \citep{Zelikman2022star} and TRICE \citep{phan2023training} at $q{=}1$. Recently, \citet{tajwar2026maximum} interpolate between RL and maximum likelihood by truncating $\log p$, but their estimator is zero when no sample succeeds, as at cold start, and \citet{wang2025deft} apply $\ell_q$ per token in SFT. On the theory side, exact policy gradient for tabular softmax policies converges at an $O(\nicefrac{1}{t})$ rate whose constants depend on the initialization, and can slow down near suboptimal deterministic policies \citep{mei2020a,mei2020}; \cref{thm:cold-start-lower} lower-bounds the gradient-flow escape time at $q{=}0$ by $\Omega(\nicefrac{1}{p_0})$ for any differentiable model with a bounded score. A full survey is in \cref{sec:related-work}.

\section{Conclusion}

\label{sec:conclusion}

$J_Q$ smooths SFT-then-RLVR into a single commitment parameter $q$, recovering the pipeline as a stepwise $q{=}1 \to q{=}0$ schedule. Its dual factorization yields GARL, whose prior-sampling amplification breaches GRPO's $\Omega(1/p_0)$ cold-start bottleneck, and PAFT, whose posterior-sampling attenuation connects $J_Q$ to EM and to SFT. We find that GARL compares favorably against GRPO and RB-RLOO, especially at cold start, and that PAFT learns faster than matched-reward GRPO from prompted inputs but does not reliably finish ahead of it. We discuss limitations and future directions in \cref{sec:limitations}.

\bibliography{example_paper}
\bibliographystyle{abbrvnat}

\newpage
\appendix
\renewcommand{\theHtheorem}{A\arabic{section}.\arabic{theorem}}
\renewcommand{\theHproposition}{A\arabic{section}.\arabic{theorem}}
\renewcommand{\theHlemma}{A\arabic{section}.\arabic{theorem}}
\renewcommand{\theHcorollary}{A\arabic{section}.\arabic{theorem}}
\renewcommand{\theHdefinition}{A\arabic{section}.\arabic{theorem}}
\renewcommand{\theHassumption}{A\arabic{section}.\arabic{theorem}}
\renewcommand{\theHremark}{A\arabic{section}.\arabic{theorem}}
\renewcommand{\theHequation}{A\arabic{section}.\arabic{equation}}
\section{Extended Related Work}

\label{sec:related-work}

\paragraph{$q$-log losses and continua.}
The Tsallis $q$-logarithm originates in non-extensive statistical mechanics \citep{tsallis1988possible}; escort distributions were studied by \citet{beck1993thermodynamics}. \citet{ferrari2010maximum} introduced maximum $L_q$-likelihood (MLqE), which reweights the score by $f(X;\theta)^{1-q}$, trading a small loss of asymptotic efficiency for outlier robustness, and showed that its population minimizer is the escort distribution of order $\nicefrac{1}{q}$ (\cref{thm:escort-minimizer} is the categorical case); the PAFT gradient \cref{eq:posterior-gradient} is the marginal-likelihood analog of this weighted score. Generalized cross-entropy \citep{zhang2018generalized} and $\alpha$-loss \citep{sypherd2022tunable} are reparameterizations of $\ell_q$ for classification with noisy labels; \citet{zhang2018generalized} motivate generalized cross-entropy by the slow training of mean absolute error ($q{=}0$) relative to cross-entropy ($q{=}1$), a trade-off our cold-start and noise-fitting rates make precise. Concurrently, \citet{wang2025deft} apply the deformed-log family at the \emph{token} level for SFT; their token-level gate $p^\alpha$ is the single-token specialization of our example-level $P_\vtheta^{-q}$, but their $p$ is an exact softmax probability whereas $P_\vtheta$ is an intractable marginal. Tsallis entropy has also been used as a policy regularizer in RL \citep{lee2018sparse,Nachum2018PathCL}; we use it in the loss function rather than as a policy regularizer. \citet{tajwar2026maximum} concurrently propose MaxRL, an RL-to-ML continuum via Maclaurin truncation of $\log p$; their estimator is unbiased for the truncated objective but exactly zero when no sample succeeds, while GARL targets the true $q$-loss and always has nonzero gradient since $w_m > 0$.

\paragraph{RL--MLE bridges and latent-variable training for reasoning.}
The RL-as-inference connection \citep{Levine2018ReinforcementLA,norouzi2016reward,guu2017language} treats MLE and RL as distinct frameworks; we embed them as endpoints of a single continuously parameterized family. R\'{e}nyi variational inference \citep{li2016renyi} provides a complementary continuum that tightens the ELBO toward $-\log P_\vtheta$, the target $J_Q$ shares at $q{=}1$. On the latent-variable side, RLVR and GRPO \citep{guo2025deepseek,shao2024deepseekmath} optimize expected reward; STaR \citep{Zelikman2022star} bootstraps reasoning by generating and filtering rationales; TRICE \citep{phan2023training} and CoVRL \citep{wen2026coupledvariationalreinforcementlearning} are ELBO-based variational methods at the $q{=}1$ pole (TRICE via MCMC-EM; CoVRL via composite prior-posterior with hybrid sampling); SPG \citep{ding17} samples from a reward-tilted proposal $q_\vtheta(\vz \mid \vx, \vy) \propto p_\vtheta(\vz \mid \vx)\exp(R(\vz \mid \vy))$ for cold-start sequence-level RL, coinciding with the posterior under $\log$-likelihood reward. At $q{=}1$, PAFT recovers SPG's gradient and TRICE's EM gradient update over posterior samples; CoVRL further hybridizes PAFT (posterior) with GARL (prior, IWAE) via composite sampling. STaR's rejection-sampling strategy is a hard-acceptance variant of PAFT's importance resampling (\cref{sec:endpoint-recovery}). The $J_Q$ continuum extends these with the $\Omega(\nicefrac{1}{p_0}) \to \Theta(\log(\nicefrac{1}{p_0}))$ separation across $q$ and the dual factorization through GARL.

\paragraph{Convergence of policy gradient.}
\citet{mei2020} prove that exact policy gradient for tabular softmax policies converges at an $O(\nicefrac{1}{t})$ rate, with constants that depend on the problem and on the smallest probability of an optimal action during training, which they bound in terms of its initial value. \citet{mei2020a} show that such trajectories are drawn toward suboptimal deterministic policies, a ``softmax gravity well'' that makes progress depend strongly on the initialization, and that softmax cross-entropy with one-hot labels converges more slowly than linearly (``softmax damping''), because near the one-hot label the squared gradient norm shrinks faster than the loss; this is the $p \to 1$ regime that our escape analysis excludes ($\delta \leq 1/2$). They mitigate both by replacing the softmax with an escort transform, which normalizes powers of the parameters' magnitudes instead of their exponentials. \Cref{thm:cold-start-lower} complements these results with a lower bound on the time to escape a small initial success probability, $\Omega(\nicefrac{1}{p_0})$ at $q{=}0$, for any differentiable model with a bounded score, and \cref{thm:cold-start-escape} shows that, when the score is also bounded away from zero, the per-example weight $P_\vtheta^{-q}$ shortens it to $\Theta(\log(\nicefrac{1}{p_0}))$ at $q{=}1$. $J_Q$ thus changes the objective rather than the parameterization, keeping the softmax of a pretrained language model; the escort distribution of \cref{thm:escort-minimizer} is the minimizer of $J_Q$, not a parameterization.

\paragraph{Gradient estimators and verifier-free training.}
GARL recovers RB-REINFORCE \citep[$q{=}0$;][]{zhou2026reinforcing} and the IWAE gradient \citep[$q{=}1$;][]{burda2016importance}. \citet{rainforth2018tighter} showed IWAE's inference-network gradient SNR shrinks as $M$ grows, motivating doubly reparameterized variants \citep{roeder2017sticking,tucker2019doubly}; our bias expansion $O(\nicefrac{q}{M P_\vtheta^q})$ exposes a related phenomenon along the $J_Q$ continuum, with intermediate $q$ balancing escape against estimator quality. \citet{zhou2026reinforcing} introduce VeriFree, the RB-REINFORCE estimator GARL extends; while Rao--Blackwellization reduces variance, \cref{sec:experiments} shows it does not address the cold-start escape bottleneck. Both GARL and PAFT are verifier-free across the $J_Q$ continuum. Finally, \citet{yue2025does} observed that RLVR narrows the reasoning capability boundary during training; our framework attributes this to mode-seeking at $q{=}0$ (\cref{thm:endpoint-behavior}), with PAFT (\cref{sec:paft}) an empirically more stable alternative to GARL with prompted inputs (\cref{sec:experiments}).

\section{Proofs for \texorpdfstring{\cref{sec:background}}{Setup and Background}: Setup and Background}
\label{sec:proofs-background}

\begin{proposition}[RLVR connection]
\label{thm:rlvr-connection}
Under the conditional model of \cref{sec:background} and exact-match reward $R(\hat{\vy},\vy^*)=\mathbb{I}(\hat{\vy} = \vy^*)$, the expected reward equals $\mathbb{E}_\Dc[P_{\vtheta}]$; consequently $J_0(\vtheta) = 1 - \mathbb{E}_\Dc[P_{\vtheta}]$, and minimizing $J_0$ is equivalent to maximizing expected reward.
\end{proposition}
\begin{proof}
    For a fixed example $(\vx^*, \vy^*)$,
    \begin{align*}
        & \E[\substack{\vz \sim p_{\vtheta}(\cdot \mid \vx^*), \hat{\vy} \sim p_{\vtheta}(\cdot \mid \vx^*, \vz)}]{R(\hat{\vy},\vy^*)}  = \sum_{\substack{\vz \in \Zc, \vy \in \Yc}} \Big[ p_{\vtheta}(\vz \mid \vx^*) \cdot p_{\vtheta}(\vy \mid \vx^*, \vz) \mathbb{I}({\vy} = \vy^*) \Big].
    \end{align*}
    The indicator picks out the correct output, giving
    \begin{align*}
        \E[\substack{\vz \sim p_{\vtheta}(\cdot \mid \vx^*),\\ \hat{\vy} \sim p_{\vtheta}(\cdot \mid \vx^*, \vz)}]{R(\hat{\vy},\vy^*)} &= \sum_{{\vz \in \Zc}} p_{\vtheta}(\vz \mid \vx^*) p_{\vtheta}(\vy^* \mid \vx^*, \vz) \\
        &= {P_{\vtheta}}.
    \end{align*}
    Taking an expectation over training examples from $\Dc$, we have
    \begin{align*}
        \E[\substack{(\vx^*, \vy^*) \sim \Dc \\ \vz \sim p_{\vtheta}(\cdot \mid \vx^*), \\ \hat{\vy} \sim p_{\vtheta}(\cdot \mid \vx^*, \vz)}]{R(\hat{\vy},\vy^*)} &= \E[{(\vx^*, \vy^*) \sim \Dc}]{P_{\vtheta}}.
    \end{align*}
\end{proof}

\begin{proposition}[SFT on annotated rationales is an ELBO for the $q{=}1$ objective]
\label{thm:sft-elbo}
Fix an example $(\vx^*, \vy^*)$ with $P_\vtheta > 0$, and let $A$ be a distribution over rationales $\vz \in \Zc$ (an \defn{annotation distribution}) supported where $p_\vtheta(\vz \mid \vx^*, \vy^*) > 0$. Then
\begin{align}
    \underbrace{\E[\vz \sim A]{-\log p_\vtheta(\vz, \vy^* \mid \vx^*)}}_{\text{SFT loss on } (\vx^*, \vz, \vy^*)}
    \;=\;
    \underbrace{-\log P_\vtheta}_{q{=}1 \text{ loss}}
    \;+\;
    \mathrm{KL}\big(A \,\big\|\, p_\vtheta(\cdot \mid \vx^*, \vy^*)\big)
    \;+\;
    \mathrm{H}(A),
    \label{eq:sft-elbo}
\end{align}
where the entropy $\mathrm{H}(A)$ does not depend on $\vtheta$. Consequently, (i) the SFT loss is an upper bound on the $q{=}1$ loss $-\log P_\vtheta$; equivalently, $\E[\vz \sim A]{\log p_\vtheta(\vz, \vy^* \mid \vx^*)} + \mathrm{H}(A)$ is an evidence lower bound (ELBO) on $\log P_\vtheta$ with variational distribution $A$. (ii) The SFT gradient is $-\nabla_\vtheta \log P_\vtheta + \nabla_\vtheta \mathrm{KL}(A \,\|\, p_\vtheta(\cdot \mid \vx^*, \vy^*))$; the KL gap vanishes iff $A = p_\vtheta(\cdot \mid \vx^*, \vy^*)$, in which case the SFT gradient equals $-\nabla_\vtheta \log P_\vtheta$ exactly.
\end{proposition}
\begin{proof}
Fix $(\vx^*, \vy^*)$ with $P_\vtheta > 0$ and write $\pi_\vtheta(\vz) \triangleq p_\vtheta(\vz \mid \vx^*, \vy^*) = p_\vtheta(\vz, \vy^* \mid \vx^*) / P_\vtheta$ for the model's posterior over rationales. The support condition on $A$ makes every term below finite. For each $\vz$ in the support of $A$, the chain rule gives
\begin{align*}
    -\log p_\vtheta(\vz, \vy^* \mid \vx^*)
    = -\log P_\vtheta - \log \pi_\vtheta(\vz)
    = -\log P_\vtheta + \log\frac{A(\vz)}{\pi_\vtheta(\vz)} - \log A(\vz).
\end{align*}
Taking expectations under $\vz \sim A$,
\begin{align*}
    \E[\vz \sim A]{-\log p_\vtheta(\vz, \vy^* \mid \vx^*)}
    = -\log P_\vtheta + \mathrm{KL}(A \,\|\, \pi_\vtheta) + \mathrm{H}(A),
\end{align*}
which is \cref{eq:sft-elbo}; $\mathrm{H}(A) = -\E[\vz \sim A]{\log A(\vz)}$ does not involve $\vtheta$.

\emph{(i)} Since $\mathrm{KL}(A \,\|\, \pi_\vtheta) \geq 0$ and $\mathrm{H}(A) \geq 0$, the SFT loss is at least $-\log P_\vtheta$. Rearranging \cref{eq:sft-elbo},
\begin{align*}
    \log P_\vtheta
    = \E[\vz \sim A]{\log p_\vtheta(\vz, \vy^* \mid \vx^*)} + \mathrm{H}(A) + \mathrm{KL}(A \,\|\, \pi_\vtheta)
    \geq \E[\vz \sim A]{\log p_\vtheta(\vz, \vy^* \mid \vx^*)} + \mathrm{H}(A),
\end{align*}
the evidence lower bound with variational distribution $A$; the same inequality follows from Jensen's inequality applied to $\log P_\vtheta = \log \E[\vz \sim A]{p_\vtheta(\vz, \vy^* \mid \vx^*) / A(\vz)}$.

\emph{(ii)} Differentiating \cref{eq:sft-elbo} in $\vtheta$ with $A$ fixed,
\begin{align*}
    \nabla_\vtheta\, \E[\vz \sim A]{-\log p_\vtheta(\vz, \vy^* \mid \vx^*)}
    &= -\nabla_\vtheta \log P_\vtheta + \nabla_\vtheta \mathrm{KL}(A \,\|\, \pi_\vtheta),
    \\
    \nabla_\vtheta \mathrm{KL}(A \,\|\, \pi_\vtheta)
    &= -\E[\vz \sim A]{\nabla_\vtheta \log \pi_\vtheta(\vz)}.
\end{align*}
By Gibbs' inequality the KL term vanishes iff $A = \pi_\vtheta$. When $A = \pi_\vtheta$ at the current iterate, $-\E[\vz \sim \pi_\vtheta]{\nabla_\vtheta \log \pi_\vtheta(\vz)} = -\nabla_\vtheta \sum_{\vz \in \Zc} \pi_\vtheta(\vz) = 0$ (for countably infinite $\Zc$, by dominated convergence), so the SFT gradient equals $-\nabla_\vtheta \log P_\vtheta$ exactly. Equivalently, $\nabla_\vtheta \log P_\vtheta = \E[\vz \sim \pi_\vtheta]{\nabla_\vtheta \log p_\vtheta(\vz, \vy^* \mid \vx^*)}$ (Fisher's identity), the posterior form of the gradient used in \cref{eq:posterior-gradient}.
\end{proof}

\begin{remark}[The single-annotation case and cold-start escape]
\label{rem:sft-single-annotation}
With one annotated rationale per example, $A = \delta_{\vz^*}$, we have $\mathrm{H}(A) = 0$ and $\mathrm{KL}(\delta_{\vz^*} \,\|\, \pi_\vtheta) = -\log \pi_\vtheta(\vz^*)$, so \cref{eq:sft-elbo} reads $-\log p_\vtheta(\vz^*, \vy^* \mid \vx^*) = -\log P_\vtheta - \log p_\vtheta(\vz^* \mid \vx^*, \vy^*)$: the standard SFT loss is the $q{=}1$ loss plus a penalty on the posterior probability the model assigns to the annotated rationale. Under gradient flow, the SFT loss is a log-loss on the joint probability $p_\vtheta(\vz^*, \vy^* \mid \vx^*)$. \Cref{thm:cold-start-lower,thm:cold-start-escape} are stated for gradient flow on $-\ln_q$ of any differentiable probability with bounded score and therefore apply to it verbatim with $P_\vtheta$ replaced by the joint; since the joint is at most $P_\vtheta$, driving it to $\delta$ drives the success probability to at least $\delta$. This is the sense in which the SFT stage of the SFT-then-RLVR pipeline escapes cold start at the $q{=}1$ rate (\cref{sec:convergence-rates}). The same penalty commits the model to the annotation whether or not it is correct, so annotation errors are fit rather than filtered.
\end{remark}

\begin{remark}[The no-annotation case: EM and PAFT at $q{=}1$]
\label{rem:sft-no-annotation}
The KL gap in \cref{eq:sft-elbo} measures how far the annotations sit from the model's own posterior over rationales; it is zero, and the ELBO is tight, exactly when $A = \pi_\vtheta$. Taking $A$ to be the model posterior at the current iterate is therefore the natural choice when no annotations exist: by \cref{thm:sft-elbo}(ii), an SFT step on posterior-sampled rationales follows the exact $q{=}1$ gradient $-\nabla_\vtheta \log P_\vtheta$ (Fisher's identity; \cref{eq:posterior-gradient}). With $A$ frozen at an annotator, SFT is a generalized-EM M-step whose E-step never updates; with $A$ re-tightened to the posterior at every iterate, it is the EM gradient update \citep{dempster1977maximum}, which PAFT recovers at $q{=}1$ (\cref{sec:paft}, \cref{sec:endpoint-recovery}): SFT with self-generated rationales. This is how the $q{=}1$ pole is realized in our no-annotation experiments (\cref{sec:experiments}); PAFT at $q < 1$ keeps the same posterior-sampled SFT step and attenuates it by $P_\vtheta^{1-q}$ (\cref{thm:q-gradient-geometry}).
\end{remark}

\section{Proofs for \texorpdfstring{\cref{sec:loss-landscape}}{Loss Landscape}: Loss Landscape}
\label{sec:proofs-landscape}

\begin{proposition}[Dispersion penalty]
\label{thm:dispersion-penalty}
For $q > 0$, $J_Q(\vtheta, q) \geq -\ln_q(\bar{P})$, where $\bar{P} \triangleq \E[(\vx^*, \vy^*) \sim \Dc]{P_{\vtheta}}$ is the mean success probability across examples, with equality if and only if $P_{\vtheta}$ is constant across all examples in $\Dc$.
\end{proposition}
\begin{proof}
For $q > 0$, the function $h_q(u) = -\ln_q(u) = \frac{1 - u^{1-q}}{1 - q}$ is strictly convex on $(0,1]$, since $h_q''(u) = q\, u^{-q-1} > 0$. Applying Jensen's inequality:
\begin{align*}
    J_Q(\vtheta, q)
    &=
    \E[(\vx^*, \vy^*) \sim \Dc]{h_q(P_{\vtheta})} \\
    &\geq
    h_q\!\bigl(\E[(\vx^*, \vy^*) \sim \Dc]{P_{\vtheta}}\bigr)
    =
    -\ln_q(\bar{P}),
\end{align*}
with equality iff $P_{\vtheta}$ is constant across all examples.
\end{proof}

\EscortMinimizer*
\begin{proof}
\emph{Case $q \in (0,1]$.}
Since $h_q$ is strictly convex for $q > 0$, the objective is strictly convex on the interior of $\Delta_K$, and the minimizer is unique. Since all $\alpha_j > 0$, the minimizer lies in the interior (any boundary point has infinite loss for $q = 1$ and suboptimal loss for $q < 1$), so we can use Lagrange multipliers for the equality constraint $\sum_j \theta_j = 1$:
\begin{align*}
    -\alpha_j \theta_j^{-q} - \lambda = 0
    \quad \Longrightarrow \quad
    \alpha_j \theta_j^{-q} = \mu
    \quad \text{for all } j,
\end{align*}
where $\mu \triangleq -\lambda > 0$. Solving: $\theta_j = (\alpha_j / \mu)^{1/q}$. The constraint $\sum_j \theta_j = 1$ yields $\mu^{1/q} = \sum_k \alpha_k^{1/q}$, giving $\theta_j^*(q) = \alpha_j^{1/q} / \sum_k \alpha_k^{1/q}$ as in \cref{thm:escort-minimizer}.

\emph{Case $q = 0$.} The objective $J_Q(\vtheta, 0) = 1 - \sum_j \alpha_j \theta_j$ is linear, minimized at any vertex $e_j$ with $j \in \argmax_k \alpha_k$.
\end{proof}

\begin{corollary}[Endpoint behavior and monotone sharpening]
\label{thm:endpoint-behavior}
Under the categorical model:
\begin{enumerate}
    \item \textbf{Density-estimation pole} ($q = 1$): $\theta_j^*(1) = \alpha_j$. The model exactly recovers the data distribution.
    \item \textbf{Exploitation pole} ($q \to 0^+$): assuming a unique mode $j^* = \argmax_k \alpha_k$, $\theta_j^*(q) \to \mathbb{I}(j = j^*)$. The model concentrates all mass on the most frequent output.
    \item \textbf{Monotone sharpening}: for $0 < q' < q \leq 1$ and $\alpha_j > \alpha_k$, $\theta_j^*(q')/\theta_k^*(q') > \theta_j^*(q)/\theta_k^*(q)$.
\end{enumerate}
\end{corollary}
\begin{proof}
Part (1): $\nicefrac{1}{q} = 1$. Part (2): $(\alpha_j/\alpha_{j^*})^{1/q} \to 0$ for $j \neq j^*$. Part (3): $\theta_j^*/\theta_k^* = (\alpha_j/\alpha_k)^{1/q}$, increasing in $\nicefrac{1}{q}$.
\end{proof}

\begin{corollary}[Propriety]
\label{thm:properness}
The Tsallis $q$-logarithmic scoring rule is strictly proper if and only if $q = 1$.
\end{corollary}
\begin{proof}
By \cref{thm:escort-minimizer}, the maximizer of $\E[y \sim \alpha]{\ln_q(\theta_y)}$ is $\theta_j^* \propto \alpha_j^{1/q}$, which equals $\alpha$ iff $q = 1$. For $q \in (0, 1)$ the true distribution $\alpha$ is not even a maximizer (the rule is not proper at all), let alone the unique one.
\end{proof}

The robustness counterpart under label noise \Dash both static (where the escort minimizer concentrates) and dynamic (how fast the model gets there) \Dash appears in \cref{sec:noise-fitting}.

\section{Proofs for \texorpdfstring{\cref{sec:convergence-rates}}{Commitment Dynamics under Gradient Flow}: Commitment Dynamics under Gradient Flow}
\label{sec:proofs-convergence}

\subsection{Warm-up: exact analysis on the sigmoid model}
\label{sec:sigmoid-warmup}

Before proving the general results, we work through the scalar sigmoid model $P(\theta) = \sigma(\theta) = (1 + e^{-\theta})^{-1}$ as a warm-up. This model admits exact closed-form escape times that validate the $\Theta(\cdot)$ bounds in \cref{thm:cold-start-escape}.

Under gradient flow on $\ell_q(\theta) = -\ln_q(\sigma(\theta))$, the parameter evolves as $\dot{\theta} = P(\theta)^{-q}P'(\theta)$. Since $P'(\theta) = P(\theta)(1-P(\theta))$, the chain rule gives:
\begin{align*}
    \dot{p} = [P'(\theta)]^2 P(\theta)^{-q} = p^{2-q}(1-p)^2.
\end{align*}
This is a special case of the general dynamics (\cref{eq:general-p-dynamics}) with score norm $\|s(\theta)\|^2 = (1-p)^2$, which satisfies $\|s\|^2 \in [(1-\delta)^2, 1]$ on $p \in [p_0, \delta]$ \Dash confirming the bounded score assumption.

The separable ODE gives the exact escape time:
\begin{align}
    T_q(p_0, \delta) = \int_{p_0}^{\delta} \frac{du}{u^{2-q}(1-u)^2}.
    \label{eq:exact-time}
\end{align}
We evaluate this integral using a dominant/remainder decomposition. Write $(1-u)^{-2} = 1 + r(u)$ where $r(u) = \frac{2u-u^2}{(1-u)^2}$. On $u \in [0,\delta]$ with $\delta \leq 1/2$, we have $0 \leq r(u) \leq 8u$. Substituting and distributing:
\begin{align*}
    T_q(p_0, \delta)
    =
    \underbrace{\int_{p_0}^{\delta} \frac{du}{u^{2-q}}}_{\text{dominant}}
    +
    \underbrace{\int_{p_0}^{\delta} \frac{r(u)}{u^{2-q}}\,du}_{\text{remainder}}.
\end{align*}

\emph{Case $q \in (0,1)$.} The dominant integral evaluates to $\frac{p_0^{-(1-q)} - \delta^{-(1-q)}}{1-q} = \frac{p_0^{-(1-q)}}{1-q}(1 + o(1))$. The remainder satisfies $0 \leq \int r(u)\,u^{-(2-q)}\,du \leq 8\int u^{q-1}\,du = \frac{8\delta^q}{q}$, a constant. So the remainder is negligible and $T_q = \frac{p_0^{-(1-q)}}{1-q}(1+o(1))$.

\emph{Case $q = 0$.} The dominant integral gives $\frac{1}{p_0}(1+o(1))$. The remainder is $O(\log(1/p_0))$, still negligible compared to $1/p_0$. So $T_0 = \frac{1}{p_0}(1+o(1))$.

\emph{Case $q = 1$.} The dominant integral is $\log(1/p_0) + \log\delta$. The remainder satisfies $\int r(u)\,u^{-1}\,du \leq 8(\delta-p_0) = O(1)$. So $T_1 = \log(1/p_0)(1+o(1))$.

Note that the sigmoid model yields exact $1+o(1)$ asymptotics (not just $\Theta(\cdot)$) because $\|s\|^2 = (1-p)^2 \to 1$ as $p \to 0$, so the score norm converges to a known constant. This is stronger than the general theorem, which only assumes bounded score norms.

\subsection{Proof of \texorpdfstring{\cref{thm:cold-start-lower}}{Cold-start lower bound}: Exploitation is provably slow}

\ColdStartLower*
\begin{proof}
From \cref{eq:general-p-dynamics}, $\dot{p} = p^{2-q}\|s(\vtheta)\|^2 \leq C^2\,p^{2-q}$. By the ODE comparison principle (since $u \mapsto u^{2-q}$ is nondecreasing on $(0,1]$), $p(t) \leq p^*(t)$ where $p^*$ solves $\dot{p}^* = C^2 (p^*)^{2-q}$ with $p^*(0) = p_0$. So $p$ reaches $\delta$ no sooner than $p^*$:
\begin{align*}
    T_q \geq \frac{1}{C^2}\int_{p_0}^{\delta} \frac{du}{u^{2-q}}.
\end{align*}
For $q \in [0,1)$, the integral evaluates to $\frac{p_0^{-(1-q)} - \delta^{-(1-q)}}{1-q} = \frac{p_0^{-(1-q)}}{1-q}(1+o(1))$, giving $T_q = \Omega(p_0^{-(1-q)}/(1-q))$.

For $q = 1$, the integral is $\log(\delta/p_0) = \log(1/p_0)(1+o(1))$, giving $T_1 = \Omega(\log(1/p_0))$.
\end{proof}

\subsection{Proof of \texorpdfstring{\cref{thm:cold-start-escape}}{Cold-start tight bounds}: Tight cold-start escape rates}

\ColdStartEscape*
\begin{proof}
The lower bound on time ($\Omega$) follows from \cref{thm:cold-start-lower}. For the upper bound, the additional assumption $\|s\| \geq c > 0$ gives $\dot{p} \geq c^2\,p^{2-q}$; by the ODE comparison principle, $p(t) \geq p_*(t)$ where $p_*$ solves $\dot{p}_* = c^2 (p_*)^{2-q}$, so $p$ reaches $\delta$ no later than $p_*$:
\begin{align*}
    T_q \leq \frac{1}{c^2}\int_{p_0}^{\delta} \frac{du}{u^{2-q}}.
\end{align*}
This integral evaluates to $\frac{p_0^{-(1-q)}}{1-q}(1+o(1))$ for $q \in [0,1)$ and $\log(1/p_0)(1+o(1))$ for $q = 1$. Combined with the lower bound, $T_q = \Theta(p_0^{-(1-q)}/(1-q))$ for $q < 1$ and $T_1 = \Theta(\log(1/p_0))$.

\emph{Speedup ratio.} For $q < q' < 1$: $T_q/T_{q'} = \Theta(p_0^{-(q'-q)}) \to \infty$. For $q < 1$ and $q' = 1$: $T_q/T_1 = \Theta(p_0^{-(1-q)}/\log(1/p_0)) \to \infty$.
\end{proof}

\subsection{The lower bound under interference and discrete steps}
\label{sec:beyond-single-example}

\Cref{thm:cold-start-lower} tracks one example under gradient flow. Both simplifications are conservative for this lower bound: interference between examples and discrete, choppy minibatch updates cannot beat the single-example speed limit $\dot p \leq C^2 p^{2-q}$ of \cref{eq:general-p-dynamics}.

Consider $N$ training examples with success probabilities $P_i \triangleq P_\vtheta(\vx_i^*, \vy_i^*)$, scores $s_i \triangleq \nabla_\vtheta \log P_i$ satisfying $\|s_i\| \leq C$ along the trajectory, and $p_{\max} \triangleq \max_i P_i$. By \cref{thm:q-gradient-geometry}, a gradient step on $J_Q$ restricted to a minibatch $B \subseteq \{1, \dots, N\}$ moves $\vtheta$ along $\vd_B \triangleq \frac{1}{|B|}\sum_{j \in B} P_j^{-q} \nabla_\vtheta P_j$. Since $\|\nabla_\vtheta P_j\| = P_j \|s_j\| \leq C P_j$ and $P_j^{1-q} \leq p_{\max}^{1-q}$ for $q \leq 1$, Cauchy--Schwarz gives, for every example $i$,
\begin{align}
    \nabla_\vtheta P_i \cdot \vd_B
    \;\leq\; \frac{1}{|B|}\sum_{j \in B} P_j^{-q}\, \|\nabla_\vtheta P_i\|\, \|\nabla_\vtheta P_j\|
    \;\leq\; C^2 p_{\max}^{2-q},
    \qquad
    \|\vd_B\| \;\leq\; C\, p_{\max}^{1-q}.
    \label{eq:interference-bound}
\end{align}
Each example's push on $P_i$ is capped by the size of its own gradient, which is small at cold start: alignment between examples can redistribute gradient mass but not create it. The bound is attained only when all gradients are positively parallel with scores at the cap; any misalignment lowers it.

\begin{proposition}[Interference cannot beat the single-example bound]
\label{thm:multi-example-lower}
Under full-batch gradient flow $\dot\vtheta = \vd_{\{1, \dots, N\}}$ on $J_Q$ with $p_{\max}(0) < \delta \leq 1/2$, no example reaches success probability $\delta$ before time $\frac{1}{C^2}\int_{p_{\max}(0)}^{\delta} u^{-(2-q)}\,du$, the bound of \cref{thm:cold-start-lower} at the easiest example's initial probability.
\end{proposition}
\begin{proof}
Every example attaining the maximum at time $t$ has $\dot P_i \leq C^2 p_{\max}^{2-q}$ by \cref{eq:interference-bound}. As a maximum of finitely many $C^1$ functions, $p_{\max}$ has upper right Dini derivative $D^+ p_{\max}(t) = \max_{i : P_i(t) = p_{\max}(t)} \dot P_i(t) \leq C^2 p_{\max}(t)^{2-q}$. Hence $G(t) \triangleq \int_{p_{\max}(0)}^{p_{\max}(t)} u^{-(2-q)}\,du$ is continuous with $D^+ G \leq C^2$, so $G(t) \leq C^2 t$; at the first time $p_{\max}(t) = \delta$, $G(t) = \int_{p_{\max}(0)}^{\delta} u^{-(2-q)}\,du$.
\end{proof}

\begin{proposition}[Discrete minibatch steps cannot beat it either]
\label{thm:discrete-lower}
Let $\vtheta_{k+1} = \vtheta_k + \eta\,\vd_{B_k}$ for a step size $\eta > 0$ and any sequence of minibatches $B_k \subseteq \{1, \dots, N\}$, with exact per-example gradients. Suppose each $\nabla_\vtheta P_i$ is $L$-Lipschitz, let $p_0$ be the value of $p_{\max}$ at $\vtheta_0$, and let $\eta \leq 2 p_0^{q}/L$. Then the first step $K$ at which some example reaches $\delta$ satisfies $K \geq \frac{1}{2 \eta C^2}\int_{p_0}^{\delta} u^{-(2-q)}\,du$: the elapsed time $\eta K$ is at least half the gradient-flow bound. At $q = 0$ the step-size condition can be dropped by replacing the factor $2$ with $1 + L\eta/2$.
\end{proposition}
\begin{proof}
Write $p_k \triangleq p_{\max}(\vtheta_k)$. $L$-smoothness and \cref{eq:interference-bound} give, for every example, $P_i(\vtheta_{k+1}) \leq P_i(\vtheta_k) + \eta\, \nabla_\vtheta P_i \cdot \vd_{B_k} + \tfrac{L}{2}\eta^2 \|\vd_{B_k}\|^2 \leq p_k + f(p_k)$ with $f(u) \triangleq \eta C^2 u^{2-q}\bigl(1 + \tfrac{L\eta}{2} u^{-q}\bigr)$. The iterates need not increase monotonically, so we track the running maximum $u_k \triangleq \max_{j \leq k} p_j$: since $f$ is nonnegative and nondecreasing for $q \leq 1$, $u_{k+1} \leq u_k + f(u_k)$, and since $u_k \geq p_0$, the step-size condition gives $f(u_k) \leq 2 \eta C^2 u_k^{2-q}$. As $u^{-(2-q)}$ is decreasing, $\int_{u_k}^{u_{k+1}} u^{-(2-q)}\,du \leq (u_{k+1} - u_k)\, u_k^{-(2-q)} \leq 2 \eta C^2$. Summing over $k < K$, with $u_0 = p_0$ and $u_K \geq \delta$, gives $\int_{p_0}^{\delta} u^{-(2-q)}\,du \leq 2 \eta C^2 K$.
\end{proof}

Because the argument bounds the progress of each step, it holds for every minibatch sequence, including non-monotone (choppy) trajectories, where dips in $p$ only delay escape. It concerns the lower bound only: the matching upper bound of \cref{thm:cold-start-escape} does not extend to many examples, since conflicting examples can slow escape at every $q$. Instead, we provide empirical evidence of achievability at dataset scale (\cref{tab:cold-start}).

\subsection{Noise-fitting rate under symmetric label noise}
\label{sec:noise-fitting}

The cold-start escape rates (\cref{thm:cold-start-lower,thm:cold-start-escape}) measure how fast the model commits to correct supervision under the $J_Q$ amplification $P_\vtheta^{-q}$. The symmetric question is how fast the model commits to \emph{incorrect} supervision: the same amplification drives both, giving the following dynamical formulation of robustness under label noise.

\paragraph{Noise-contamination setup.}
We work with a two-label categorical model, chosen to expose the mechanism in the simplest possible setting. For a single input $\vx^*$, the model predicts one of two labels $\{c, k\}$ with probabilities $p_\vtheta(c \mid \vx^*) = p$ and $p_\vtheta(k \mid \vx^*) = 1 - p$. We instantiate the parameterization with the sigmoid $p = \sigma(\theta)$ used in \cref{sec:sigmoid-warmup}, under which $s \triangleq \nabla_\theta \log p = \tilde p$ and $\|s\|^2 = \tilde p^2$. The target label is \emph{corrupted}: with probability $1 - \epsilon$ it equals the clean value $c$, and with probability $\epsilon \in (0, 1/2)$ it flips to the noise value $k$, giving $\tilde\alpha = (1-\epsilon, \epsilon)$. The restriction to two labels is cosmetic: in the $N$-label categorical model with symmetric noise $\tilde\alpha = (1-\epsilon)\alpha + \epsilon\cdot\mathrm{Unif}$, conditioning on the two-subset $\{j^*, k\}$ containing the clean mode $j^*$ and any fixed wrong label $k$ reduces to this binary setting.

Let $p(t) = p_\vtheta(c \mid \vx^*)$ denote the clean-mode probability under gradient flow on $J_Q(\vtheta) = \mathbb{E}_{y \sim \tilde\alpha}[\ell_q(p_\vtheta(y \mid \vx^*))]$, and let $\tilde p(t) = 1 - p(t)$ denote the noise contamination. The cold-start analysis (\cref{thm:cold-start-escape}) assumed a non-vanishing score $\|s\| \geq c_* > 0$; the analogous lower bound fails near $p = 1$, where the sigmoid score vanishes linearly in $\tilde p$, so we substitute the actual scaling $\|s\|^2 = \tilde p^2$ rather than treating $\|s\|$ as a constant.

\paragraph{The escort asymptote.}
Differentiating $J(p) = (1-\epsilon)\ell_q(p) + \epsilon\ell_q(1-p)$ gives $J'(p) = -(1-\epsilon)p^{-q} + \epsilon\tilde p^{-q}$. Gradient flow on the sigmoid yields
\begin{align}
\dot{\tilde p} = -\dot p = [\epsilon\tilde p^{-q} - (1-\epsilon)(1-\tilde p)^{-q}]\,p^2\,\tilde p^2.
\label{eq:noise-fitting-dynamics}
\end{align}
For $q > 0$, the dynamics have a unique stable equilibrium at
\begin{align}
\tilde p_*(q) \;=\; (\epsilon/(1-\epsilon))^{1/q}\,(1 + o(1)) \quad \text{as } \epsilon \to 0,
\label{eq:noise-asymptote}
\end{align}
obtained by solving $J'(p) = 0$ ($\|s\|^2$ cancels at equilibrium, so $\tilde p_*(q)$ does not depend on the parameterization). This equilibrium coincides with the static escort minimizer from \cref{thm:escort-minimizer} applied to $\tilde\alpha$: at $q = 1$, $\tilde p_*(1) = \epsilon$ (the model fits observed noise exactly); as $q \to 0$, $\tilde p_*(q) \to 0$ (the model concentrates on the clean mode, paralleling \cref{thm:endpoint-behavior}). The escort is both where $J_Q$ is minimized (static) and where gradient flow converges (dynamic).

The noise-to-clean ratio $\epsilon\tilde p^{-q} / [(1-\epsilon)(1-\tilde p)^{-q}]$ is monotone decreasing in $\tilde p$ on $(0, 1)$: it diverges as $\tilde p \to 0$ (noise term dominates near the clean mode), equals $1$ at $\tilde p = \tilde p_*(q)$ (equilibrium), and vanishes as $\tilde p \to 1$. So for $\tilde p \ll \tilde p_*(q)$ \Dash the regime of small noise contamination \Dash the noise term in \cref{eq:noise-fitting-dynamics} dominates by an arbitrarily large factor. This drives the asymptotic scaling.

\begin{proposition}[Noise-fitting rate]
\label{thm:noise-fitting}
Fix $q \in (0, 1]$. Under the setup above, starting from $\tilde p(0) = \tilde p_0$ with $\tilde p_0 \ll \tilde p_*(q)$, the time $T_q^{\mathrm{noise}}(\tilde p_0)$ to reach a fixed target $\eta$ (with $\tilde p_0 \ll \eta \leq \tilde p_*(q)$, $\eta$ independent of $\tilde p_0$) satisfies, as $\tilde p_0 \to 0$:
\begin{align}
T_q^{\mathrm{noise}}(\tilde p_0) = \Theta\!\left(\frac{\tilde p_0^{-(1-q)}}{(1-q)\,\epsilon}\right) \text{ for } q \in (0,1), \qquad T_1^{\mathrm{noise}}(\tilde p_0) = \Theta\!\left(\frac{\log(1/\tilde p_0)}{\epsilon}\right).
\label{eq:noise-fitting-rate}
\end{align}
The speedup ratio for $0 < q < q' \leq 1$ diverges: $T_q^{\mathrm{noise}}(\tilde p_0) / T_{q'}^{\mathrm{noise}}(\tilde p_0) = \Theta(\tilde p_0^{-(q'-q)}) \to \infty$ as $\tilde p_0 \to 0$. At $q = 0$, adopting the convention $\tilde p^0 \equiv 1$, the dynamics \cref{eq:noise-fitting-dynamics} reduce to $\dot{\tilde p} = -(1-2\epsilon)\,p^2\,\tilde p^2 < 0$ everywhere (for $\epsilon < 1/2$), so any positive $\tilde p_0$ decays monotonically toward 0: $T_0^{\mathrm{noise}}(\tilde p_0) = \infty$ for any target $\eta > \tilde p_0$.
\end{proposition}

\begin{proof}
By the noise-to-clean monotonicity established above, for any $K > 1$ there exists $\tilde p_K(q) = K^{-1/q}\,\tilde p_*(q)(1 + o(1))$ such that for $\tilde p \leq \tilde p_K$, the noise term in \cref{eq:noise-fitting-dynamics} exceeds $K$ times the clean term. Combined with $p = 1 - \tilde p \to 1$ as $\tilde p \to 0$ and $\|s\|^2 = \tilde p^2$:
\begin{align*}
\dot{\tilde p} \;\in\; \bigl[(1 - \tfrac{1}{K})\,\epsilon\,\tilde p^{2-q}\,(1 + o(1)),\; \epsilon\,\tilde p^{2-q}\bigr].
\end{align*}
Fix any $K > 1$ (e.g., $K = 2$). Separating variables, $\tilde p^{q-2}\,d\tilde p = \Theta(\epsilon)\,dt$. For $q \in (0, 1)$, integrating from $\tilde p_0$ to $\eta$ with $\tilde p_0 \ll \eta \leq \tilde p_K(q)$ gives
\begin{align*}
\frac{\tilde p_0^{-(1-q)} - \eta^{-(1-q)}}{1-q} = \Theta(\epsilon\,T),
\end{align*}
so $T_q^{\mathrm{noise}}(\tilde p_0) = \Theta(\tilde p_0^{-(1-q)}/((1-q)\epsilon))$ as $\tilde p_0 \to 0$. (The integral from exactly $\tilde p_0 = 0$ diverges for $q \leq 1$, so a positive starting contamination is required.) For $q = 1$, $\dot{\tilde p} = \Theta(\epsilon\,\tilde p)$ gives $\tilde p(t) = \tilde p_0\,\exp(\Theta(\epsilon\,t))$, so $T_1^{\mathrm{noise}}(\tilde p_0) = \Theta(\log(\eta/\tilde p_0)/\epsilon) = \Theta(\log(1/\tilde p_0)/\epsilon)$. The speedup ratio $T_q/T_{q'} = \Theta(\tilde p_0^{-(q'-q)})$ diverges for $q < q' \leq 1$ as $\tilde p_0 \to 0$.
\end{proof}

\paragraph{Structural parallel with cold-start escape.}
\cref{thm:cold-start-escape} gives $T_q^{\mathrm{escape}}(p_0) = \Theta(p_0^{-(1-q)}/(1-q))$ for $q < 1$ and $\Theta(\log(1/p_0))$ at $q{=}1$, with speedup ratio $\Theta(p_0^{-(q'-q)})$. \cref{thm:noise-fitting} gives $T_q^{\mathrm{noise}}(\tilde p_0) = \Theta(\tilde p_0^{-(1-q)}/((1-q)\epsilon))$ and $\Theta(\log(1/\tilde p_0)/\epsilon)$, with speedup ratio $\Theta(\tilde p_0^{-(q'-q)})$ \Dash the \emph{exact dual}: same exponent in the small starting probability ($p_0$ for cold-start escape from clean, $\tilde p_0$ for noise-fitting escape from corruption), with the noise rate $\epsilon$ as the only additional rate factor. The same $P_\vtheta^{-q}$ amplification accelerates commitment to clean and corrupted supervision by the same multiplicative factor. Static mode-seeking (\cref{thm:endpoint-behavior}) is recovered as the $t \to \infty$ limit of \cref{eq:noise-fitting-dynamics}: $\tilde p(t) \to \tilde p_*(q) \to 0$ as $q \to 0$.

\subsection{Near-optimality convergence (supplementary result)}

\begin{proposition}[Near-optimality convergence is $q$-independent]
\label{thm:near-optimality}
Suppose that near optimality, $\|s(\vtheta)\|^2$ depends on $\vtheta$ only through $P_\vtheta$ (i.e., $\|s(\vtheta)\|^2 = h(P_\vtheta)$ for some function $h$). Then for $\epsilon_0 \ll 1$ and $\epsilon_1 < \epsilon_0$, the time to improve from $P_\vtheta = 1 - \epsilon_0$ to $P_\vtheta = 1 - \epsilon_1$ satisfies
\begin{align*}
    T_q(1-\epsilon_0, 1-\epsilon_1) = T_{q'}(1-\epsilon_0, 1-\epsilon_1)\bigl(1 + O(\epsilon_0)\bigr)
\end{align*}
for all $q, q' \in [0,1]$. That is, the convergence time is the same for all members of the $J_Q$ family up to a correction that vanishes as $\epsilon_0 \to 0$.
\end{proposition}
\begin{proof}
Write $\epsilon = 1-p$ with $\epsilon \ll 1$. From \cref{eq:general-p-dynamics}, $\dot{\epsilon} = -(1-\epsilon)^{2-q}\,\|s(\vtheta)\|^2 < 0$. Since $\epsilon$ decreases over time, the convergence time from $\epsilon_0$ to $\epsilon_1$ is:
\begin{align*}
    T_q = \int_{\epsilon_1}^{\epsilon_0} \frac{d\epsilon}{(1-\epsilon)^{2-q}\,\|s(\vtheta)\|^2}.
\end{align*}
For any $q, q' \in [0,1]$, the integrands of $T_q$ and $T_{q'}$ differ by the factor $(1-\epsilon)^{q-q'}$. We bound this factor on $\epsilon \in [\epsilon_1, \epsilon_0]$ with $\epsilon_0 \ll 1$. Using the Taylor expansion $\log(1-\epsilon) = -\epsilon - \epsilon^2/2 - \cdots$:
\begin{align*}
    \log(1-\epsilon)^{q-q'}
    &= (q-q')\log(1-\epsilon) \\
    &= (q-q')\bigl(-\epsilon - \tfrac{\epsilon^2}{2} - \cdots\bigr).
\end{align*}
Since $|q-q'| \leq 1$:
\begin{align*}
    \bigl|\log(1-\epsilon)^{q-q'}\bigr|
    \leq \epsilon + \tfrac{\epsilon^2}{2} + \cdots
    = O(\epsilon).
\end{align*}
Exponentiating and using $e^x = 1 + x + O(x^2) = 1 + O(\epsilon)$ for $x = O(\epsilon)$, we get $(1-\epsilon)^{q-q'} = 1 + O(\epsilon)$. Since $\epsilon \leq \epsilon_0$ on $[\epsilon_1, \epsilon_0]$, the integrands of $T_q$ and $T_{q'}$ differ by a multiplicative $1 + O(\epsilon_0)$ factor, giving $T_q / T_{q'} = 1 + O(\epsilon_0)$.
\end{proof}

\section{Proofs and Pseudocode for \texorpdfstring{\cref{sec:practical-mc-estimates}}{Monte Carlo Estimators}: Monte Carlo Estimators}
\label{sec:proofs-estimators}

\BiasExpansion*
\begin{proof}
We write
\begin{align*}
    \mu_w \triangleq \E{w_m} = P_{\vtheta},
    \qquad
    \mu_g \triangleq \E{g_m} = \nabla_{\vtheta}\ell_0(\vtheta;\vx^*,\vy^*).
\end{align*}
Define the smooth map
\begin{align*}
    f(a,b) \triangleq b\,a^{-q},
\end{align*}
for $a>0$. Then
\begin{align*}
    \hat{\nabla}_{\vtheta}\ell_q(q,\vtheta;\vx^*,\vy^*,M)
    =
    f(\bar w_M,\bar g_M),
\end{align*}
while the target gradient is
\begin{align*}
    \nabla_{\vtheta}\ell_q(\vtheta,q;\vx^*,\vy^*)
    =
    f(\mu_w,\mu_g)
    =
    \mu_g\,\mu_w^{-q}.
\end{align*}

Almost sure convergence follows from the Strong Law of Large Numbers, since $\bar w_M \to \mu_w$ and $\bar g_M \to \mu_g$ almost surely, and $f$ is continuous at $(\mu_w,\mu_g)$ because $\mu_w=P_{\vtheta}>0$.

For the bias expansion, we exploit the linearity of $f$ in its second argument: $f(a,b)=b\,a^{-q}$, so
\begin{align*}
    f(\bar w_M,\bar g_M)
    &= \bar g_M \cdot h(\bar w_M)
    \\
    &= \underbrace{\mu_g\,h(\bar w_M)}_{\text{first piece}} + \underbrace{(\bar g_M - \mu_g)\,h(\bar w_M)}_{\text{second piece}},
\end{align*}
where $h(a) \triangleq a^{-q}$ is a scalar function whose derivatives $h^{(k)}(a) = (-q)(-q\!-\!1)\cdots(-q\!-\!k\!+\!1)\,a^{-(q+k)}$ depend only on $a$.

\paragraph{First piece.} Expand $h(\bar w_M)$ to \emph{third} order around $\mu_w$, with $h'(a)=-qa^{-q-1}$, $h''(a)=q(q+1)a^{-q-2}$, $h'''(a)=-q(q+1)(q+2)a^{-q-3}$:
\begin{align*}
    h(\bar w_M)
    &= \underbrace{h(\mu_w)}_{\E{\cdot} = \mu_w^{-q}}
    + \underbrace{h'(\mu_w)(\bar w_M - \mu_w)}_{\E{\cdot} = 0}
    + \underbrace{\tfrac{1}{2}h''(\mu_w)(\bar w_M - \mu_w)^2}_{\E{\cdot} = \frac{q(q+1)}{2M}\mu_w^{-q-2}\mathbf{Var}(w_m)}
    \\
    &\quad+ \underbrace{\tfrac{1}{6}h'''(\mu_w)(\bar w_M - \mu_w)^3}_{\E{\cdot} = O(M^{-2})\text{ via }\kappa_3/M^2}
    + \underbrace{R_M^{(1)}}_{\text{4th-order}}.
\end{align*}
Therefore:
\begin{align*}
    \mu_g\,\E{h(\bar w_M)}
    &= \mu_g\,\mu_w^{-q}
    + \frac{q(q+1)}{2M}\,\mu_g\,\mu_w^{-q-2}\,\mathbf{Var}(w_m)
    \\
    &\quad+ O(M^{-2})
    + \mu_g\,\E{R_M^{(1)}}.
\end{align*}

\paragraph{Second piece.} The factor $(\bar g_M - \mu_g) = O_p(M^{-1/2})$, so a \emph{second}-order expansion of $h(\bar w_M)$ suffices. Multiplying $(\bar g_M - \mu_g)$ by each term of the expansion and taking expectations:
\begin{align*}
    &\E{(\bar g_M - \mu_g)\,h(\bar w_M)}
    \\
    &=
    \underbrace{h(\mu_w)\,\E{\bar g_M - \mu_g}}_{= 0}
    +
    \underbrace{h'(\mu_w)\,\E{(\bar g_M - \mu_g)(\bar w_M - \mu_w)}}_{= -\frac{q}{M}\mu_w^{-q-1}\mathbf{Cov}(g_m, w_m)}
    \\
    &\quad+
    \underbrace{\tfrac{1}{2}h''(\mu_w)\,\E{(\bar g_M - \mu_g)(\bar w_M - \mu_w)^2}}_{= O(M^{-2})\text{ via i.i.d.\ expansion}}
    +
    \underbrace{\E{R_M^{(2)}}}_{\text{3rd-order remainder}}.
\end{align*}
For the cross moment, expand $\E{(\bar g_M - \mu_g)(\bar w_M - \mu_w)^2} = M^{-3}\sum_{i,j,k}\E{(g_i - \mu_g)(w_j - \mu_w)(w_k - \mu_w)}$. By independence, the only nonzero index pattern is $i = j = k$ (all others vanish because $\E{g_i - \mu_g} = 0$ or $\E{w_j - \mu_w} = 0$). The $M$ surviving terms give $\E{(g_m - \mu_g)(w_m - \mu_w)^2}/M^2 = O(M^{-2})$, since $|(w_m - \mu_w)^2| \leq 1$ and $\E{\|g_m\|} < \infty$ (Assumption~2). The remainder has the form $R_M^{(2)} = (\bar g_M - \mu_g) \cdot O(|\bar w_M - \mu_w|^3)$.

\paragraph{Combining.} Adding the two pieces and substituting $\mu_w=P_{\vtheta}$, $\mu_g=\nabla_{\vtheta}\ell_0$, $\nabla_{\vtheta}\ell_1 = \nicefrac{\nabla_{\vtheta}\ell_0}{P_{\vtheta}}$:
\begin{align}
    \mathbb{E}\!\left[
        \hat{\nabla}_{\vtheta}\ell_q(q,\vtheta;\vx^*,\vy^*,M)
    \right]
    & =
    \nabla_{\vtheta}\ell_q(\vtheta,q;\vx^*,\vy^*)
    \nonumber \\
    &\quad+
    \frac{q}{M P_{\vtheta}^{q+1}} \cdot \left[
        \frac{q+1}{2}
        \nabla_{\vtheta}\ell_1(\vtheta;\vx^*,\vy^*)\,\mathbf{Var}(w_m)
        -
        \mathbf{Cov}(g_m,w_m)
    \right]     \label{eq:combined-bias} \\
    & \quad +
    \E{R_M}, \nonumber
\end{align}
where $R_M = \mu_g R_M^{(1)} + R_M^{(2)}$.

\paragraph{Remainder bound.} Write $\E{R_M} = \E{R_M \cdot \mathbf{1}_A} + \E{R_M \cdot \mathbf{1}_{A^c}}$ where $A = \{\bar w_M \geq P_\vtheta/2\}$.

\emph{On $A$.} The derivatives of $h$ are bounded on $\{a \geq P_\vtheta/2\}$: $|h^{(k)}(a)| \leq C_k$.

For $R_M^{(1)}$ (the \emph{fourth}-order scalar remainder), the integral form gives $|R_M^{(1)}| \leq C_4 |\bar w_M - \mu_w|^4$ on $A$. Since $w_m \in [0{,}1]$, $\E{|\bar w_M - \mu_w|^4} = O(M^{-2})$, so $\E{|R_M^{(1)}| \cdot \mathbf{1}_A} = O(M^{-2})$.

For $R_M^{(2)} = (\bar g_M - \mu_g) \cdot O(|\bar w_M - \mu_w|^3)$ on $A$ (the \emph{third}-order remainder from the second piece, a vector quantity), Cauchy--Schwarz gives $\E{\|R_M^{(2)}\| \cdot \mathbf{1}_A} \leq C_3\,\sqrt{\E{\|\bar g_M - \mu_g\|^2}}\,\sqrt{\E{(\bar w_M - \mu_w)^6}} = O(M^{-1/2})\,O(M^{-3/2}) = O(M^{-2})$, using Assumption~2 and the boundedness of $w_m$.

\emph{On $A^c$.} Assumption~3 gives $\bar w_M \geq \epsilon > 0$, so $|h(\bar w_M)| \leq \epsilon^{-q}$ everywhere and $\|f(\bar w_M, \bar g_M)\| \leq \epsilon^{-q}\,\|\bar g_M\|$. Therefore $\|R_M\| \leq \|f(\bar w_M, \bar g_M)\| + \|T_M\| \leq C\,\epsilon^{-q}\,(1 + \|\bar g_M\|)$, where $T_M$ collects the (bounded) Taylor terms. Again by Cauchy--Schwarz,
\begin{align*}
    \E{\|R_M\| \cdot \mathbf{1}_{A^c}}
    &\leq C\,\epsilon^{-q}\,\sqrt{\E{(1 + \|\bar g_M\|)^2}}\,\sqrt{P(A^c)}.
\end{align*}
The first factor is $O(1)$ by Assumption~2. For the second, since $w_m \in [0,1]$ are i.i.d.\ with mean $P_\vtheta$, Hoeffding's inequality with $t = P_\vtheta/2$ gives $P(A^c) = P(\bar w_M - P_\vtheta \leq -P_\vtheta/2) \leq \exp(-M P_\vtheta^2/2)$. Thus $\E{\|R_M\| \cdot \mathbf{1}_{A^c}}$ decays faster than any polynomial in $M$.

Combining: $\E{R_M} = O(M^{-2})$, so the leading-order bias is the explicit formula above.

\paragraph{Bound on the bracketed coefficient.} In \cref{eq:combined-bias}, the prefactor $\nicefrac{q}{(M P_\vtheta^{q+1})}$ has $P_\vtheta^{-(q+1)}$ scaling, but the bracket $[\nicefrac{(q+1)}{2}\,\nabla_\vtheta\ell_1\,\mathbf{Var}(w_m) - \mathbf{Cov}(g_m, w_m)]$ scales as $O(P_\vtheta)$, so one factor of $P_\vtheta$ cancels. Specifically:
\begin{itemize}[noitemsep]
    \item $\mathbf{Var}(w_m) \leq \mathbb{E}[w_m^2] \leq \mathbb{E}[w_m] = P_\vtheta$ since $w_m \in [0, 1]$.
    \item $\nabla_\vtheta\ell_1 = -\nabla_\vtheta \log P_\vtheta = -s$ is bounded under the bounded-score assumption used in \cref{thm:cold-start-lower}.
    \item Under bounded per-trajectory score $\|\nabla_\vtheta \log p_\vtheta(\vz, \vy^* \mid \vx^*)\| \leq C'$ (which follows from bounded weights and Lipschitz activations), $\|g_m\| \leq C' w_m$, and Cauchy--Schwarz gives $\|\mathbf{Cov}(g_m, w_m)\| \leq \sqrt{\mathbf{Var}(g_m)\,\mathbf{Var}(w_m)} \leq \sqrt{C'^2 P_\vtheta \cdot P_\vtheta} = O(P_\vtheta)$.
\end{itemize}
Hence the bracket is bounded by $\nicefrac{(q+1)}{2}\,O(P_\vtheta) + O(P_\vtheta) = O(P_\vtheta)$ (the $\nicefrac{(q+1)}{2}$ multiplier is bounded by $1$ for $q \in [0, 1]$ and absorbs into the constant), and the leading-order bias is $\nicefrac{q}{(M P_\vtheta^{q+1})} \cdot O(P_\vtheta) = O\!\left(\nicefrac{q}{(M P_\vtheta^q)}\right)$, yielding \cref{eq:bias-expansion}. The bias scales with the same $P_\vtheta^{-q}$ exponent as the cold-start amplification factor.
\end{proof}

\subsection{RLOO control variate derivation}
\label{sec:rloo-derivation}

We derive the RLOO estimator~\eqref{eq:rloo-estimator} from the plug-in estimator~\eqref{eq:mc-estimator}. Using the chain rule, $g_m$ from \eqref{eq:def-gm} decomposes into a score-function term and a pathwise term:
\begin{align}
    g_m
    &=
    -\,w_m\,\nabla_{\vtheta}\log p_{\vtheta}(\vz^{(m)} \mid \vx^*)
    -
    \nabla_{\vtheta} w_m.
    \label{eq:gm-decomp}
\end{align}
Substituting into the plug-in estimator isolates the score-function component:
\begin{align}
    & \hat{\nabla}^{\text{plug-in}}_{\vtheta}\ell_q
    =
    \frac{1}{M}\sum_{m=1}^M \left[
        \frac{-w_m}{(\bar w_M)^q} \nabla_{\vtheta}\log p_{\vtheta}(\vz^{(m)} \mid \vx^*)
        -
        \frac{\nabla_{\vtheta} w_m}{(\bar w_M)^q}
    \right].
    \label{eq:mc-estimator-decomp}
\end{align}
Since $\mathbb{E}[\nabla_{\vtheta}\log p_{\vtheta}(\vz^{(m)} \mid \vx^*)] = 0$, we can subtract any baseline from the score-function coefficient $-w_m/(\bar w_M)^q$ without changing the expected value, provided the baseline does not depend on $\vz^{(m)}$.

We use a leave-one-out approximation. Let $\bar w_{\neg m} = \frac{1}{M-1}\sum_{j \neq m} w_j$. Replacing $w_m$ with $\bar w_{\neg m}$ in the coefficient, the batch mean collapses to $\bar w_{\neg m}$, giving a surrogate coefficient of $-(\bar w_{\neg m})^{1-q}$. Subtracting this baseline yields the RLOO estimator
\begin{equation}
    \hat{\nabla}^{\mathrm{RLOO}}_{\vtheta}\ell_q
    =
    \frac{1}{M}\sum_{m=1}^M \Bigg[
        - \underbrace{\biggl( \frac{w_m}{(\bar w_M)^q} - (\bar w_{\neg m})^{1-q} \biggr)}_{\text{centered weight}}
         \cdot \nabla_{\vtheta}\log p_{\vtheta}(\vz^{(m)} \mid \vx^*)
        -
        \frac{\nabla_{\vtheta} w_m}{(\bar w_M)^q}
    \Bigg].
    \label{eq:rloo-estimator}
\end{equation}

\paragraph{Endpoint recovery.}
At $q = 0$, the centered weight evaluates to $w_m - \bar w_{\neg m}$, and the score-function term becomes $-(w_m - \bar w_{\neg m})\,\nabla_{\vtheta}\log p_{\vtheta}(\vz^{(m)} \mid \vx^*)$, exactly recovering the REINFORCE leave-one-out (RLOO) estimator standard in RLVR. At $q = 1$, the centered weight is $w_m/\bar w_M - 1$; since $\sum_{m=1}^M (w_m/\bar w_M - 1) = 0$, this acts as a self-normalizing baseline that strictly centers the importance weights across the batch.

\begin{proposition}[RLOO bias preservation]
\label{thm:rloo-bias}
Under the assumptions of \cref{thm:mc-bias-expansion}, the RLOO estimator~\eqref{eq:rloo-estimator} satisfies the same bias expansion as the plug-in estimator~\eqref{eq:mc-estimator}.
\end{proposition}

\begin{proof}
The RLOO estimator~\eqref{eq:rloo-estimator} differs from the plug-in estimator~\eqref{eq:mc-estimator-decomp} by subtracting $(\bar w_{\neg m})^{1-q}$ from the score-function coefficient $w_m / (\bar w_M)^q$ for each sample $m$. Denoting $s_m = \nabla_{\vtheta}\log p_{\vtheta}(\vz^{(m)} \mid \vx^*)$, the difference in expectations is
\begin{align*}
    \Delta &= \frac{1}{M}\sum_{m=1}^M \E{(\bar w_{\neg m})^{1-q}\, s_m}.
\end{align*}
Since $\bar w_{\neg m} = \frac{1}{M-1}\sum_{j \neq m} w_j$ is a function of $\{\vz^{(j)}\}_{j \neq m}$ only, and $s_m = \nabla_{\vtheta}\log p_{\vtheta}(\vz^{(m)} \mid \vx^*)$ is a function of $\vz^{(m)}$ only, the independence of the i.i.d.\ samples gives
\begin{align*}
    \E{(\bar w_{\neg m})^{1-q}\, s_m}
    &= \E{(\bar w_{\neg m})^{1-q}} \cdot \underbrace{\E{s_m}}_{= \, 0} = 0,
\end{align*}
where $\E{s_m} = \E[\vz \sim p_{\vtheta}]{\nabla_{\vtheta}\log p_{\vtheta}(\vz \mid \vx^*)} = 0$ is the standard score-function identity. Therefore $\Delta = 0$ and the two estimators have identical expectations for every $M$.
\end{proof}

\subsection{Endpoint recovery}
\label{sec:endpoint-recovery}

\begin{proposition}[Endpoint recovery for GARL and PAFT]
\label{thm:endpoint-recovery}
Fix a supervised example $(\vx^*, \vy^*)$ with $P_{\vtheta} > 0$.
\begin{enumerate}
    \item \textbf{GARL at $q = 0$} recovers Rao--Blackwellized REINFORCE \citep{williams1992simple,zhou2026reinforcing}:
    \begin{align*}
        \hat{\nabla}_{\vtheta}\ell_q\big|_{q=0}
        = \bar g_M
        = \frac{1}{M}\sum_{m=1}^M \bigl(-w_m\,\nabla_{\vtheta}\log p_{\vtheta}(\vz^{(m)}, \vy^* \mid \vx^*)\bigr),
    \end{align*}
    which is unbiased for $\nabla_{\vtheta}\ell_0$ by \cref{eq:gm-unbiased}. Each $g_m$ marginalizes out the output $\vy$ given $\vz^{(m)}$ analytically via $w_m = p_{\vtheta}(\vy^* \mid \vx^*, \vz^{(m)})$, rather than relying on a sampled output and binary reward.

    \item \textbf{GARL at $q = 1$} recovers the IWAE gradient estimator \citep{burda2016importance}, a self-normalized importance sampling (SNIS) estimator for $\nabla_\vtheta \log P_\vtheta$:
    \begin{align*}
        \hat{\nabla}_{\vtheta}\ell_q\big|_{q=1}
        = \frac{\bar g_M}{\bar w_M}
        = \frac{\sum_m w_m\,(-\nabla_{\vtheta}\log p_{\vtheta}(\vz^{(m)}, \vy^* \mid \vx^*))}{\sum_m w_m}.
    \end{align*}

    \item \textbf{PAFT at $q = 0$} reduces to posterior-resampled SFT scaled by $P_{\vtheta}$:
    \begin{align*}
        \hat{\nabla}_{\mathrm{PAFT}}\big|_{q=0}
        = -\bar w_M \cdot \frac{1}{K}\sum_{k=1}^{K} \nabla_{\vtheta}\log p_{\vtheta}(\vz^{(r_k)}, \vy^* \mid \vx^*).
    \end{align*}
    The factor $\bar w_M \approx P_{\vtheta}$ downweights hard instances so aggressively that this endpoint is overly conservative in practice. Unlike the other three endpoints, it does not correspond to a standard method.

    \item \textbf{PAFT at $q = 1$} recovers the EM gradient update with E-step posterior samples \citep{dempster1977maximum} / TRICE \citep{phan2023training}:
    \begin{align*}
        \hat{\nabla}_{\mathrm{PAFT}}\big|_{q=1}
        = -\frac{1}{K}\sum_{k=1}^{K} \nabla_{\vtheta}\log p_{\vtheta}(\vz^{(r_k)}, \vy^* \mid \vx^*).
    \end{align*}
    The attenuation vanishes: $(\bar w_M)^{1-1} = 1$, so all instances contribute equally, and the gradient is uniform SFT on approximate posterior samples.
\end{enumerate}
\end{proposition}
\begin{proof}
Each case follows by substituting $q = 0$ or $q = 1$ into the GARL estimator~\eqref{eq:mc-estimator} or PAFT estimator~\eqref{eq:paft-estimator} and simplifying $(\bar w_M)^0 = 1$.
\end{proof}

\subsection{PAFT bias and variance}
\label{sec:paft-bias-proof}

\begin{proposition}[PAFT has the same bias as GARL]
\label{thm:paft-bias}
Under the assumptions of \cref{thm:mc-bias-expansion},
$\mathbb{E}[\hat\nabla_{\mathrm{PAFT}}] = \mathbb{E}[\hat\nabla_{\mathrm{GARL}}]$
for all $M$. In particular, the PAFT estimator inherits the same leading bias expansion as in \cref{eq:bias-expansion}, simplifying to $O(\nicefrac{q}{M P_{\vtheta}^q})$ under bounded marginal and per-trajectory scores.
\end{proposition}
\begin{proof}
Conditional on the prior samples $\mathrm{pool} = \{(\vz^{(m)}, w_m)\}_{m=1}^M$, the factor $(\bar w_M)^{1-q}$ is deterministic. The importance-resampled average satisfies
\begin{align*}
    \mathbb{E}\!\left[\frac{1}{K}\sum_{k=1}^K f(\vz^{(r_k)}) \;\middle|\; \mathrm{pool}\right]
    = \sum_{m=1}^M \frac{w_m}{\sum_j w_j}\,f(\vz^{(m)})
    = \hat\mu_{\mathrm{SNIS}},
\end{align*}
where $f(\vz) = \nabla_{\vtheta}\log p_{\vtheta}(\vz, \vy^* \mid \vx^*)$. Therefore
\begin{align*}
    \mathbb{E}[\hat\nabla_{\mathrm{PAFT}} \mid \mathrm{pool}]
    &= -(\bar w_M)^{1-q} \cdot \hat\mu_{\mathrm{SNIS}} \\
    &= -(\bar w_M)^{1-q} \cdot \frac{\sum_m w_m f_m}{M\bar w_M} \\
    &= \frac{1}{(\bar w_M)^q} \cdot \frac{1}{M}\sum_m (-w_m f_m) \\
    &= \frac{\bar g_M}{(\bar w_M)^q}
    = \hat\nabla_{\mathrm{GARL}}.
\end{align*}
Taking outer expectations by the tower property: $\mathbb{E}[\hat\nabla_{\mathrm{PAFT}}] = \mathbb{E}[\hat\nabla_{\mathrm{GARL}}]$.
\end{proof}

\begin{proposition}[GARL has strictly lower variance than PAFT]
\label{thm:paft-variance}
Under the same setup, $\mathbf{Var}(\hat\nabla_{\mathrm{PAFT}}) \geq \mathbf{Var}(\hat\nabla_{\mathrm{GARL}})$, with equality only when $\mathbf{Var}(\hat\nabla_{\mathrm{PAFT}} \mid \mathrm{pool}) = 0$ almost surely.
\end{proposition}
\begin{proof}
By \cref{thm:paft-bias}, $\mathbb{E}[\hat\nabla_{\mathrm{PAFT}} \mid \mathrm{pool}] = \hat\nabla_{\mathrm{GARL}}$. The law of total variance gives
\begin{align*}
    \mathbf{Var}(\hat\nabla_{\mathrm{PAFT}})
    &= \mathbf{Var}\bigl(\mathbb{E}[\hat\nabla_{\mathrm{PAFT}} \mid \mathrm{pool}]\bigr)
    + \mathbb{E}\bigl[\mathbf{Var}(\hat\nabla_{\mathrm{PAFT}} \mid \mathrm{pool})\bigr] \\
    &= \mathbf{Var}(\hat\nabla_{\mathrm{GARL}})
    + \underbrace{\mathbb{E}\bigl[\mathbf{Var}(\hat\nabla_{\mathrm{PAFT}} \mid \mathrm{pool})\bigr]}_{\geq\, 0},
\end{align*}
with equality iff $\mathbf{Var}(\hat\nabla_{\mathrm{PAFT}} \mid \mathrm{pool}) = 0$ a.s. This holds when, for each pool realization, all resampled trajectories produce the same gradient \Dash e.g., when a single trajectory dominates the importance weights. In the non-degenerate case, the inequality is strict.
\end{proof}

\subsection{Pseudocode for GARL and PAFT}
\label{sec:pseudocode}

\paragraph{Update normalization.}
The per-sample weight $\nicefrac{w_m}{(\bar w_M)^q}$ (the effective reward under the RL view) has maximum $M^q$, so the centered advantage $c_m$ in \cref{eq:rloo-estimator} can range up to $M^q$ in magnitude. To keep the per-sample advantage uniformly bounded as $q$ varies, the algorithms \cref{alg:jq-gradient,alg:paft-gradient} divide by $M^q$, yielding $c_m/M^q \in [-1, 1]$. The mathematical estimators \cref{eq:rloo-estimator,eq:paft-estimator} target $\nabla_\vtheta \ell_q$ directly; the algorithm-side $1/M^q$ is equivalent to applying a $q$-independent learning rate to the bounded-advantage form (vs.\ a $q$-dependent learning rate to the unscaled form).

\begin{algorithm}[h]
\caption{GARL: per-example $J_Q$ gradient with RLOO control variate. \textit{Numerical stability}: $w_m = \prod_t p_\vtheta(y^*_t \mid \cdot)$ underflows for long $\vy^*$ in linear-space arithmetic, so $w_m$, $\bar w_M$, $\bar w_{\neg m}$, and $c_m$ should be computed in log-space (e.g., LogSumExp); the pathwise term $\nicefrac{\nabla_\vtheta w_m}{(\bar w_M)^q}$ should be implemented as $\frac{w_m}{(\bar w_M)^q}\,\nabla_\vtheta \log p_\vtheta(\vy^* \mid \vx^*, \vz^{(m)})$ (log-derivative trick), with the coefficient computed in log-space before being applied to the log-probability gradient.}
\label{alg:jq-gradient}
\begin{algorithmic}[1]
\REQUIRE Example $(\vx^*, \vy^*)$, interpolation parameter $q \in [0,1]$, number of latent samples $M \geq 2$ (for the leave-one-out baseline)
\STATE Sample latent trajectories $\vz^{(1)}, \dots, \vz^{(M)} \sim p_{\vtheta}(\cdot \mid \vx^*)$
\FOR{$m = 1, \dots, M$}
    \STATE $w_m \leftarrow p_{\vtheta}(\vy^* \mid \vx^*, \vz^{(m)})$ \hfill $\triangleright$ likelihood weight
    \STATE $\nabla_{\vtheta} w_m \leftarrow \nabla_{\vtheta}\, p_{\vtheta}(\vy^* \mid \vx^*, \vz^{(m)})$ \hfill $\triangleright$ pathwise gradient of output likelihood
\ENDFOR
\STATE $\bar w_M \leftarrow \frac{1}{M}\sum_{m=1}^M w_m$ \hfill $\triangleright$ batch mean (estimates $P_\vtheta$)
\FOR{$m = 1, \dots, M$}
    \STATE $\bar w_{\neg m} \leftarrow \frac{1}{M-1}\sum_{j \neq m} w_j$ \hfill $\triangleright$ leave-one-out mean
    \STATE $c_m \leftarrow \frac{w_m}{(\bar w_M)^q} - (\bar w_{\neg m})^{1-q}$ \hfill $\triangleright$ centered weight (RLOO baseline)
    \STATE $\hat{g}_m \leftarrow -c_m \, \nabla_{\vtheta}\log p_{\vtheta}(\vz^{(m)} \mid \vx^*) - \frac{\nabla_{\vtheta} w_m}{(\bar w_M)^q}$ \hfill $\triangleright$ score-function + pathwise terms
\ENDFOR
\STATE \textbf{return} $\hat{g} \leftarrow \frac{1}{M^q} \cdot \frac{1}{M}\sum_{m=1}^M \hat{g}_m$ \hfill $\triangleright$ per-example gradient estimate, rescaled by $1/M^q$ to bound per-sample advantage uniformly in $q$
\end{algorithmic}
\end{algorithm}

\begin{algorithm}[h]
\caption{PAFT: per-example $J_Q$ gradient via importance resampling. \textit{Numerical stability}: the resampling step should be implemented with a categorical distribution parameterized by log-weights, e.g., $\mathrm{Categorical}(\mathrm{logits} = [\log w_1, \dots, \log w_M])$, to avoid division-by-zero when all $w_m$ underflow.}
\label{alg:paft-gradient}
\begin{algorithmic}[1]
\REQUIRE Example $(\vx^*, \vy^*)$, interpolation parameter $q \in [0,1]$, prior samples $M$, resampled trajectories $K$
\STATE Sample latent trajectories $\vz^{(1)}, \dots, \vz^{(M)} \sim p_{\vtheta}(\cdot \mid \vx^*)$
\FOR{$m = 1, \dots, M$}
    \STATE $w_m \leftarrow p_{\vtheta}(\vy^* \mid \vx^*, \vz^{(m)})$ \hfill $\triangleright$ likelihood weight (same as GARL)
\ENDFOR
\STATE $\bar w_M \leftarrow \frac{1}{M}\sum_{m=1}^M w_m$ \hfill $\triangleright$ batch mean (estimates $P_\vtheta$)
\STATE Resample indices $r_1, \dots, r_K \sim \mathrm{Categorical}(w_1/\textstyle\sum_j w_j, \dots, w_M/\textstyle\sum_j w_j)$
\STATE $\hat{g} \leftarrow -\frac{(\bar w_M)^{1-q}}{M^q \, K}\sum_{k=1}^{K} \nabla_{\vtheta}\log p_{\vtheta}(\vz^{(r_k)}, \vy^* \mid \vx^*)$ \hfill $\triangleright$ attenuated SFT on coherent rationales, rescaled by $1/M^q$ for advantage bounding
\STATE \textbf{return} $\hat{g}$ \hfill $\triangleright$ rescaled per-example gradient estimate (matching Alg.\ 1's $1/M^q$ convention)
\end{algorithmic}
\end{algorithm}

\section{Additional Experimental Details}

\label{sec:experimental-details}

\paragraph{Subset construction.}
We sample subsets from Huggingface datasets: FinQA from \texttt{dreamerdeo/\allowbreak finqa}, HotPotQA from \texttt{hotpotqa/\allowbreak hotpot\_qa}, and MuSiQue from \texttt{bdsaglam/\allowbreak musique}.
We construct training, validation, and test subsets by retaining instances whose pre-tokenization input length (in characters) falls below predefined caps. The caps are $8000$, $4000$, and $10000$ characters for FinQA, HotPotQA, and MuSiQue respectively. The resulting train/val/test subset sizes are 6145/872/1132, 9067/342/343, and 9985/579/445 for the 3 datasets respectively.

\paragraph{Prompts.} Every input is a two-turn chat, a system message and a user message (\cref{fig:prompts}), rendered with the model's chat template with thinking enabled. With prompted inputs, the user message is a task instruction, a blank line, and the example; without prompts, as at cold start, it is the example alone, formatted the same way. The system message is the same in both settings. The Gemma 4 E4B instructions do not mention \texttt{<think>} tags, since Gemma delimits its thinking with special tokens. The completion that GARL and PAFT score, $\texttt{t}_\text{end}$\texttt{\textbackslash n<answer>}$\vy^*$\texttt{</answer>} (\cref{sec:experimental-setup}), does not depend on the prompt.

\begin{figure}[t]
\small
\textbf{System message} (all models, with and without prompts)
\begin{lstlisting}[style=prompt]
You are a helpful assistant.
\end{lstlisting}
\textbf{Instruction, FinQA} (Qwen 3)
\begin{lstlisting}[style=prompt]
You are given a table and surrounding texts from a financial document, followed by a question. Think step by step, and provide your brief and concise reasoning between <think> and </think> tokens. Then, provide the final answer (i.e., just one numerical value), bracketed between <answer> and </answer> immediately after.
\end{lstlisting}
\textbf{Instruction, FinQA} (Gemma 4)
\begin{lstlisting}[style=prompt]
You are given a table and surrounding texts from a financial document, followed by a question. Think step by step, and then provide the final answer (i.e., just one numerical value), bracketed between <answer> and </answer>.
\end{lstlisting}
\textbf{Instruction, HotPotQA and MuSiQue} (Qwen 3)
\begin{lstlisting}[style=prompt]
You are given a question and a set of context paragraphs. Answer the question based on the provided context. Think step by step, and provide your reasoning between <think> and </think> tokens. Then, provide your answer between <answer> and </answer>.
\end{lstlisting}
\textbf{Instruction, HotPotQA} (Gemma 4)
\begin{lstlisting}[style=prompt]
You are given a question and a set of context paragraphs. Answer the question based on the provided context. Think step by step, and then provide your answer between <answer> and </answer>.
\end{lstlisting}
\begin{minipage}[t]{0.49\linewidth}
\textbf{Example, FinQA}
\begin{lstlisting}[style=prompt]
pre_text:
row 1: {sentence}
row 2: {sentence}
...

table:
row 1: {cell} | {cell} | ...
row 2: {cell} | {cell} | ...
...

post_text:
['{sentence}', '{sentence}', ...]
question:
{question}
\end{lstlisting}
\end{minipage}\hfill
\begin{minipage}[t]{0.49\linewidth}
\textbf{Example, HotPotQA and MuSiQue}
\begin{lstlisting}[style=prompt]
Context:
[{title}] {paragraph}

[{title}] {paragraph}

...

Question: {question}
\end{lstlisting}
\end{minipage}
\caption{Model inputs, verbatim from the data files; training, validation, and test use the same templates. With prompts, the user message is the instruction, a blank line, and the example; without prompts, it is the example alone. Braces mark fields filled from each example, and $\hookrightarrow$ marks a line wrapped for display.}
\label{fig:prompts}
\end{figure}

\paragraph{Training setup.} We do not apply KL regularization to a reference policy, following the VeriFree setup~\citep{zhou2026reinforcing}; \citet{liu2025understanding} found KL does not improve performance in this regime. Per-rationale token budgets force the thinking-end token (\texttt{</think>} for Qwen, \texttt{<channel|>} for Gemma) once the budget is exhausted \citep{muennighoff2025s1simpletesttimescaling}; see \emph{Generation lengths} below. We use the AdamW optimizer \citep{loshchilov2018decoupled} for all experiments. Training batch size is $64$, and learning rate is set to $5 \times 10^{-7}$ for Qwen 3 0.6B (higher learning rate was unstable in preliminary experiments; \cref{sec:warm-start-results} also reports prompted runs at lower rates), and $1 \times 10^{-6}$ for Qwen 3 8B experiments respectively. Gemma 4 E4B runs follow the Qwen 3 0.6B protocol, including its learning rate. We train for $2$ epochs for all datasets, with a constant learning rate (no warmup or decay). Rollouts during training use temperature $1.0$ (with top-$k$/top-$p$ sampling disabled).

\paragraph{Model selection.} We evaluate on the validation sets every $50$ steps, and also at the end of training. We select the checkpoint that performs best on the \texttt{m@16} metric.

\paragraph{Generation lengths.} We cap the maximum generation lengths to be $4096$ for FinQA, $3072$ for HotPotQA, and $2048$ for MuSiQue. In addition, we allocate $128$ tokens at the end of generation for the answer.

\paragraph{Compute.} Each run uses one $8$-GPU (NVIDIA A100 80GB) machine. A training step, excluding validation, takes about $2$--$4.5$ minutes for Qwen 3 0.6B, $3$--$5$ minutes for Qwen 3 8B, and $5$--$8$ minutes for Gemma 4 E4B.

\clearpage

\section{Additional results with prompted inputs}
\label{sec:warm-start-results}

The prompted setting (\cref{sec:experimental-setup}) tests whether GARL and PAFT help when $P_\vtheta$ is not negligible and GRPO already makes progress. We conduct prompted experiments with PAFT; while PAFT is not suitable for cold-start setups, its gradient attenuation is meant to keep training stable when $| P^{-q}_{\vtheta}|$ varies greatly over training examples. PAFT (\cref{alg:paft-gradient}) uses the same $M = 32$ rollouts per prompt as GARL, resamples $K = M$ trajectories from that pool, and teacher-forces the same canonical completion; other settings follow \cref{sec:experimental-setup}. \Cref{tab:warm-start-maj16} reports test \texttt{p@1}, \texttt{p@16}, and \texttt{m@16} for Qwen 3 0.6B at the default learning rate ($5 \times 10^{-7}$), with the same best-validation checkpoint selection as the main text; \cref{tab:prompted-lr} compares GARL and matched-reward GRPO at lower learning rates on HotPotQA and MuSiQue.

\begin{table}[t]
\centering
\footnotesize
\caption{Test \texttt{p@1}, \texttt{p@16}, and \texttt{m@16} (Qwen 3 0.6B; test at the best-validation checkpoint; evaluation uses the reshaped reward for all rows). Base = un-adapted model with the same prompts. GRPO is trained with the matched reward (3-seed mean, as in \cref{tab:cold-start}) or the reshaped reward (3-seed mean on FinQA, 2-seed mean on HotPotQA, single-seed on MuSiQue). Unprompted GARL numbers are 3-seed mean from \cref{tab:cold-start}. All other rows are single-seed. GARL at $q{=}0$ recovers RB-RLOO \citep{zhou2026reinforcing}. {Prompted GARL entries at $q > 0$ on HotPotQA and MuSiQue are peak-before-collapse}: validation accuracy later falls to zero, except that the HotPotQA $q{=}0.75$ run ended at step 89, before its next evaluation, after its gradient fell to near zero. FinQA GARL is stable. After their best checkpoints, the PAFT $q{=}0.75$ runs lose up to 30\% (HotPotQA) and 41\% (MuSiQue) of their peak validation accuracy but stay far from zero. Best per column in bold.}
\label{tab:warm-start-maj16}
\begin{adjustbox}{max width=\textwidth}
\begin{tabular}{@{}l ccc ccc ccc@{}}
\toprule
& \multicolumn{3}{c}{FinQA} & \multicolumn{3}{c}{HotPotQA} & \multicolumn{3}{c}{MuSiQue} \\
\cmidrule(lr){2-4} \cmidrule(lr){5-7} \cmidrule(lr){8-10}
Method & \texttt{p@1} & \texttt{p@16} & \texttt{m@16} & \texttt{p@1} & \texttt{p@16} & \texttt{m@16} & \texttt{p@1} & \texttt{p@16} & \texttt{m@16} \\
\midrule
Base (no training, prompted) & 9.8 & 28.9 & 12.6 & 18.3 & 46.4 & 22.2 & 7.2 & 24.3 & 8.9 \\
GRPO (matched reward, prompted) & 17.6 & 45.8 & 24.8 & 48.5 & 65.7 & 51.6 & 19.8 & 43.1 & 23.4 \\
GRPO (reshaped reward, prompted) & 30.1 & 57.9 & 37.1 & 53.2 & 67.7 & 55.4 & \textbf{28.9} & 51.7 & 34.0 \\
\addlinespace
GARL ($q = 0$ (RB-RLOO), prompted) & \textbf{32.0} & 57.9 & 38.3 & 17.7 & 44.2 & 21.6 & 7.4 & 25.3 & 9.1 \\
GARL ($q = 0.25$, prompted) & 31.7 & 59.3 & \textbf{38.7} & 19.2 & 48.9 & 22.9 & 20.4 & 42.9 & 24.3 \\
GARL ($q = 0.75$, prompted) & 29.6 & 59.4 & 37.6 & 39.0 & 63.4 & 46.8 & 15.4 & 39.5 & 19.7 \\
\addlinespace
PAFT ($q = 0.25$, prompted) & 20.4 & 47.5 & 26.6 & 42.4 & 62.9 & 47.0 & 7.1 & 24.6 & 9.0 \\
PAFT ($q = 0.75$, prompted) & 20.9 & 49.9 & 28.6 & 43.5 & 64.3 & 47.9 & 13.5 & 38.4 & 17.5 \\
\midrule
GARL ($q=0.75$, unprompted) & 30.5 & \textbf{61.5} & 38.6 & \textbf{53.4} & \textbf{74.6} & \textbf{57.4} & 27.5 & \textbf{58.3} & \textbf{35.6} \\
\bottomrule
\end{tabular}
\end{adjustbox}
\normalsize
\end{table}

\paragraph{With prompted inputs, $J_Q$ estimators match or beat matched-reward GRPO on FinQA and MuSiQue; reshaped-reward GRPO is stronger on HotPotQA and MuSiQue.}
Against matched-reward GRPO, which trains toward the same success event as GARL and PAFT (\cref{sec:experimental-setup}), the best $J_Q$ checkpoint gains $+13.9$ \texttt{m@16} on FinQA (GARL, $q{=}0.25$) and $+0.9$ on MuSiQue (GARL, $q{=}0.25$, before collapse), and trails by $3.7$ on HotPotQA (PAFT, $q{=}0.75$). At a fifth of the learning rate, where GARL ($q{=}0.75$) does not collapse within two epochs, the comparison is similar: level on MuSiQue ($23.0 \pm 0.9$ vs.\ $23.4 \pm 5.0$) and $3.1$ behind on HotPotQA (\cref{tab:prompted-lr}). Reshaped-reward GRPO, which trains on the evaluation's own reward function, is much stronger with prompted inputs: GARL beats it only on FinQA, by $0.5$ to $1.6$ points ($37.6$--$38.7$ vs.\ $37.1$), and on HotPotQA and MuSiQue reshaped-reward GRPO scores above every GARL and PAFT entry ($55.4$ vs.\ at most $47.9$; $34.0$ vs.\ at most $24.3$). On FinQA, GARL is stable at every tested $q$ and \texttt{m@16} is roughly flat across $q \in [0, 0.75]$, consistent with the bias and noise-fitting costs of high $q$ (\cref{thm:mc-bias-expansion,thm:noise-fitting}) outweighing its amplification once reward is not sparse. At $q{=}0$, GARL recovers RB-RLOO, which beats matched-reward GRPO on FinQA ($+13.5$) but not on HotPotQA ($-30.0$) or MuSiQue ($-14.3$); raising $q$ lifts the best checkpoint on those two benchmarks (HotPotQA $21.6 \to 46.8$ at $q{=}0.75$; MuSiQue $9.1 \to 24.3$ at $q{=}0.25$).

\paragraph{PAFT learns faster than matched GRPO early in training, but its final accuracy is mixed.}
After 50 steps, PAFT at $q{=}0.75$ reaches validation \texttt{m@16} of $53.3$ on HotPotQA, $25.4$ on FinQA, and $17.8$ on MuSiQue, against $29.2$--$32.0$, $15.0$--$15.7$, and $10.4$--$11.3$ for the three seeds of matched-reward GRPO, which trains toward the same success event. At the best-validation checkpoint, however, PAFT is level with matched GRPO on FinQA ($+3.8$ \texttt{m@16}, within GRPO's seed standard deviation) and trails it on HotPotQA ($-3.7$) and MuSiQue ($-5.9$; \cref{tab:warm-start-maj16}; PAFT is single-seed). PAFT's validation accuracy peaks early, at step 100 on HotPotQA and 150 on FinQA and MuSiQue, and then declines as its rationales grow toward the generation cap; we relate this drift below to the teacher-forced weights $w_m$, which favor rationales that restate the answer. Reshaped-reward GRPO scores higher on all three benchmarks, but its reward encodes task knowledge that PAFT does not use.

\paragraph{Prompted GARL training destabilizes after its best checkpoint on HotPotQA and MuSiQue.}
On HotPotQA, GARL's validation accuracy falls to zero at $q{=}0.25$ (it peaks at step 50 at $30.6$, falls to $9.7$ at step 100, and is zero by step 150); the $q{=}0.75$ run ended at step 89, before its next evaluation, after its gradient norm had risen about 70-fold and then fallen to near zero as responses lengthened to about 2{,}700 tokens. On MuSiQue it falls to zero in every run that learned, between steps 200 and 250 at $q{=}0.25$ and between steps 50 and 100 at $q{=}0.75$. At $q{=}0$, GARL barely trains on either benchmark: its validation accuracy stays within two points of its starting value. In the collapses we could inspect (MuSiQue at the default learning rate, HotPotQA at half of it), responses run to the generation cap and the gradient goes to zero; where we tracked the gradient norm per step, it spikes first. Matched-reward GRPO also degrades late in some runs: on HotPotQA, validation \texttt{m@16} peaks at $58.4$--$61.8$ (steps 200--250), and two of the three seeds then fall, to $45.6$ and $24.9$ at the last evaluation; on MuSiQue, one seed drops from $23.2$ to $8.3$ (steps 150--200) before recovering, and another falls from $33.6$ to $25.5$ over its last 62 steps; on FinQA, one seed falls from $32.3$ to $3.7$ at its last evaluation, while the other two are still improving when training ends. Reshaped-reward GRPO instead improves throughout training on all three benchmarks and ends at its best checkpoint. GARL's failure is the most abrupt: on MuSiQue, validation accuracy goes from its peak to exactly zero between consecutive evaluations, 50 steps apart, and on HotPotQA at $q{=}0.25$ within two.

\paragraph{At a fifth of the learning rate, prompted GARL does not collapse within two epochs, but matched-reward GRPO barely trains.}
\Cref{tab:prompted-lr} repeats GARL ($q{=}0.75$, three seeds) and matched-reward GRPO (one seed) on HotPotQA and MuSiQue at a learning rate of $1 \times 10^{-7}$. No GARL run collapses within the two training epochs. The pre-clip gradient norm peaks at $41$--$42$ on HotPotQA and $23$--$24$ on MuSiQue across seeds, against $496$ and $1{,}109$ in the runs that collapsed at the default rate, and the mean response length stays at or below $582$ and $912$ tokens, against $2{,}839$ and $1{,}931$ (generation caps $3{,}072$ and $2{,}048$). Validation \texttt{m@16} rises through training, apart from dips of at most $2.2$ points at the last evaluation, so the best checkpoint is at or near the end of training (steps 250--282 of 282 on HotPotQA, 300--312 of 312 on MuSiQue), and it scores above the pre-collapse peak at the default rate on both benchmarks ($48.5 \pm 0.3$ vs.\ $46.8$; $23.0 \pm 0.9$ vs.\ $19.7$). At half the default rate, the HotPotQA run peaked later (step 100, $48.1$) but still collapsed before step 150. Training for four epochs at $1 \times 10^{-7}$ shows that the lower rate delays the HotPotQA collapse rather than removing it: in both seeds we ran, the gradient norm starts rising around steps 530--540 of 564, and validation \texttt{m@16} falls from $56.8$--$57.3$ at step 500 to $2.4$--$13.1$ at the last step; a single four-epoch MuSiQue run stays stable through step 624. Matched-reward GRPO at $1 \times 10^{-7}$ stays within one point of the untrained model ($22.9$ vs.\ $22.2$; $9.7$ vs.\ $8.9$). Under the exact-match reward, a step has a nonzero gradient only when at least one rollout matches the reference exactly, and $260$ of its $282$ steps on HotPotQA and $257$ of $312$ on MuSiQue had none. At the default rate this reward also starts slowly (validation \texttt{m@16} at step 50 is $29.2$--$32.0$ across the three seeds on HotPotQA, against $28.5$ at step 0) but takes off by step 100 ($44.3$--$51.3$); at a fifth of the step size it does not take off within two epochs. Of the learning rates we tried, matched-reward GRPO makes progress only at the default rate, and GARL stays stable through two epochs only at the lowest. Each at its working rate, GARL is level with matched-reward GRPO on MuSiQue ($23.0 \pm 0.9$ vs.\ $23.4 \pm 5.0$) and trails it on HotPotQA ($48.5 \pm 0.3$ vs.\ $51.6 \pm 0.7$).

\begin{table}[t]
\centering
\footnotesize
\caption{Prompted GARL ($q{=}0.75$) and matched-reward GRPO at several learning rates (Qwen 3 0.6B; test \texttt{m@16} at the best-validation checkpoint, with that step in parentheses for GARL). Rows at $5 \times 10^{-7}$ repeat \cref{tab:warm-start-maj16}. GRPO at $5 \times 10^{-7}$ and GARL at $1 \times 10^{-7}$ are 3-seed means $\pm$ sample standard deviation, with the step range spanning the seeds; all other rows are single-seed. GARL at $5 \times 10^{-7}$ and $2.5 \times 10^{-7}$ collapses after the reported checkpoint; we did not run MuSiQue at $2.5 \times 10^{-7}$. Best per benchmark in bold.}
\label{tab:prompted-lr}
\begin{tabular}{@{}llrr@{}}
\toprule
Method & Learning rate & HotPotQA & MuSiQue \\
\midrule
Base (no training, prompted) & -- & 22.2 & 8.9 \\
\addlinespace
GRPO (matched reward, prompted) & $5 \times 10^{-7}$ & \textbf{51.6} $\pm 0.7$ & \textbf{23.4} $\pm 5.0$ \\
GRPO (matched reward, prompted) & $1 \times 10^{-7}$ & 22.9 & 9.7 \\
\addlinespace
GARL ($q = 0.75$, prompted) & $5 \times 10^{-7}$ & 46.8 (50) & 19.7 (50) \\
GARL ($q = 0.75$, prompted) & $2.5 \times 10^{-7}$ & 48.1 (100) & -- \\
GARL ($q = 0.75$, prompted) & $1 \times 10^{-7}$ & 48.5 $\pm 0.3$ (250--282) & 23.0 $\pm 0.9$ (300--312) \\
\bottomrule
\end{tabular}
\normalsize
\end{table}

\paragraph{PAFT degrades, but does not destabilize.}
PAFT's validation accuracy also declines after its best checkpoint, but gradually and far from zero. At $q{=}0.75$ it falls from $55.0$ (step 100) to a low of $38.7$ on HotPotQA ($-30\%$ relative) and from $25.1$ (step 150) to $14.7$ on MuSiQue ($-41\%$), ending at $47.7$ and $18.3$; on FinQA the drop from peak is $20\%$ at $q{=}0.25$ and $5\%$ at $q{=}0.75$. Its responses also run to the generation cap on HotPotQA and MuSiQue, yet its accuracy stays well above zero. Under best-validation selection at the default learning rate, PAFT leads GARL only on HotPotQA ($47.9$ vs.\ $46.8$); GARL's pre-collapse checkpoint scores higher on MuSiQue ($24.3$ vs.\ $17.5$), and GARL leads on FinQA at every $q$. PAFT's advantage at that rate is therefore robustness rather than a higher score: it degrades instead of collapsing, which matters when no validation set is available for checkpoint selection.

\paragraph{What drives PAFT's drift.}
PAFT weights a rationale by the teacher-forced likelihood $w_m = p_\vtheta(\vy^* \mid \vx^*, \vz^{(m)})$, which is high when the rationale states the answer, and its weights are never negative, so nothing pushes other rationales down. On HotPotQA at $q{=}0.75$, PAFT's training rationales (the roughly 1 in 64 that the training log prints) grow from a median of $310$ thinking tokens to $2{,}807$ by step 200; they mention the gold answer $4.1$ times on average at the start and $50$ times by the end, when $84\%$ of their sentences repeat an earlier one. Late in training, $87$--$91\%$ of them are cut off by the thinking budget, after which the answer is still teacher-forced. The rationales of matched-reward GRPO stay under $500$ thinking tokens (median, every seed), mention the answer at most $11$ times on average, and have under $2\%$ repeated sentences. \Cref{tab:paft-variants} tests this reading with three single-seed variants. Setting $w_m = 0$ for rationales cut off by the budget removes the late collapse (validation \texttt{m@16} stays at $51$--$52$ through step 282) but does not raise the peak, and additionally capping $w_m$ at $0.5$, so that prompts the model already solves stop favoring rationales that push $w_m$ toward $1$, changes little. Removing the gradient on the teacher-forced answer tokens, so that PAFT trains only the rationale, isolates the incentive: rationales shrink to a one-line guess of the answer (median $17$ thinking tokens by step 100), and validation accuracy falls to near zero by step 150 as the model stops writing well-formed answer tags. The drift is also not a sampling failure. In a second run of PAFT ($q{=}0.75$, HotPotQA) that logs the resampling weights, their effective sample size $(\sum_m w_m)^2 / \sum_m w_m^2$, averaged over prompts, is $66\%$ of $M$ up to the validation peak (steps 51--150) and $60\%$ after it (steps 201--282), and the share of prompts below $0.1M$ stays near $10\%$ ($10\%$ and $9\%$), while the mean response length grows from about $730$ to $2{,}600$ tokens and validation \texttt{m@16} falls from $55.6$ (step 150) to $42.9$ (step 282). Finite-$M$ resampling error therefore does not grow as PAFT drifts.

\begin{table}[t]
\centering
\footnotesize
\caption{PAFT ($q{=}0.75$) variants on HotPotQA (Qwen 3 0.6B, prompted, one seed each; test at the best-validation checkpoint). \emph{Cut off}: $w_m = 0$ for rationales that hit the thinking budget. \emph{Cap}: $w_m$ capped at $0.5$ before resampling. \emph{Rationale only}: no gradient on the teacher-forced answer tokens, run for 150 steps. Matched-reward GRPO: 3-seed mean, with validation \texttt{m@16} averaged over seeds.}
\label{tab:paft-variants}
\begin{tabular}{@{}l ccccc ccc@{}}
\toprule
& \multicolumn{5}{c}{Validation \texttt{m@16} at step} & \multicolumn{3}{c}{Test} \\
\cmidrule(lr){2-6} \cmidrule(lr){7-9}
Method & 50 & 100 & 150 & 250 & 282 & \texttt{p@1} & \texttt{p@16} & \texttt{m@16} \\
\midrule
PAFT & 53.3 & 55.0 & 51.2 & 38.7 & 47.7 & 43.5 & 64.3 & 47.9 \\
\quad + cut off & 52.9 & 54.0 & 52.9 & 51.9 & 50.9 & 42.7 & 63.2 & 47.5 \\
\quad + cut off, cap & 50.7 & 54.5 & 52.8 & 54.9 & 45.2 & 42.1 & 67.2 & 48.6 \\
\quad + cut off, rationale only & 37.3 & 35.5 & 0.5 & --- & --- & 25.1 & 57.3 & 29.2 \\
\midrule
GRPO (matched reward) & 30.3 & 48.9 & 56.9 & 55.8 & 43.7 & 48.5 & 65.7 & 51.6 \\
\bottomrule
\end{tabular}
\normalsize
\end{table}

\clearpage

\clearpage

\section{Training Dynamics Figures}
\label{sec:additional-figures}

\begin{figure}[h]
\centering
\includegraphics[width=\linewidth]{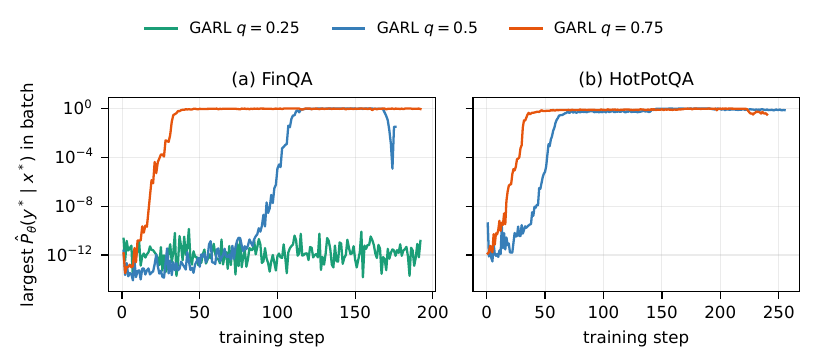}
\caption{Cold-start escape on Gemma 4 E4B, one seed per run, on (a) FinQA and (b) HotPotQA: a lower bound on the largest per-prompt Monte Carlo estimate of $P_\vtheta(\vy^* \mid \vx^*)$ in each training batch ($\log$ scale), which reflects the easiest prompt in the batch. With $w_m = p_\vtheta(\vy^* \mid \vx^*, \vz^{(m)})$, the estimate for rollout $m$'s prompt is the leave-one-out mean $\bar w_{\neg m}$. The runs did not log $w_m$, so we recover this from the logged minimum advantage in each batch, $\min c_m / M^q$: since $c_m \geq -(\bar w_{\neg m})^{1-q}$, with equality when $w_m = 0$, the plotted $|\min c_m|^{1/(1-q)}$ is at most the largest $\bar w_{\neg m}$ in the batch. GARL at $q{=}1$ is omitted because its advantages depend only on ratios of the $w_m$. On both benchmarks, $q{=}0.75$ escapes (its estimate rises to near $1$) before $q{=}0.5$, although on HotPotQA held-out accuracy rose first at $q{=}0.5$ (by step 150, against step 200 at $q{=}0.75$; evaluations every 50 steps), and both exceed prompted matched-reward GRPO in \cref{tab:cold-start}; $q{=}0.25$, plotted on FinQA only, does not escape within the training budget on either benchmark, and its accuracy stays at zero. This is the ordering that \cref{sec:convergence-rates} predicts for gradient flow from a small $p_0$: the escape time decreases with $q$ ($T_q \geq T_{q'}$ for $q < q'$, and $T_q / T_{q'} \to \infty$ as $p_0 \to 0$ by \cref{thm:cold-start-escape}), and by \cref{thm:cold-start-lower} it is at least of order $p_0^{-(1-q)}/(1-q)$, which is large at low $q$.}
\label{fig:gemma-escape}
\end{figure}

\paragraph{Collapse after escape.}
On Gemma 4 E4B MuSiQue, both runs that escaped later collapsed: at $q{=}0.75$, validation \texttt{m@16} reached $58.1$ at step 250 and fell to $0.3$ by step 300, as the pre-clip gradient norm rose from about $80$ to over $10^4$ and $90\%$ of rationales ran to the thinking budget; at $q{=}1$ it peaked at $60.9$ (step 150) and fell to $20.6$ by step 250. Best-validation selection picks the pre-collapse checkpoint, so \cref{tab:cold-start} is unaffected. To test whether the collapse is determined by the state at escape, we resumed the $q{=}0.75$ run from its step-250 checkpoint three times: with the weights only at learning rates $5 \times 10^{-7}$ and $10^{-7}$, and with the weights and Adam state at $5 \times 10^{-7}$. All three stayed stable through step 312 (pre-clip gradient norm below $80$) and reached validation \texttt{m@16} of $61.4$--$62.2$ (test $59.8$--$61.2$). The resumed runs see a different order of training batches after step 250 than the original run, so its collapse is not determined by the step-250 weights or optimizer state. A full-state restart of the Gemma 4 E4B HotPotQA $q{=}0.75$ run from step 200, which also saw a different batch order, did collapse, so whether a run collapses after escape varies between runs. \cref{tab:cold-start} reports the uninterrupted runs.

\clearpage

\section{Limitations and future directions}
\label{sec:limitations}
\label{sec:future-directions}

\paragraph{Limitations.}
The majority of experiments, and all multi-seed results, use Qwen 3 0.6B; Qwen 3 8B and Gemma 4 E4B runs are single-seed, and all reported runs use a fixed $q$. The cold-start theorems are scale-agnostic, and the cold-start ordering replicates at Qwen 3 8B on all three benchmarks and on Gemma 4 E4B on all three benchmarks (\cref{sec:experiments}). GARL can become unstable at our default learning rate, with prompted inputs and after escaping cold start (\cref{sec:additional-figures}); with prompted inputs, a fifth of that rate avoided the collapse within two epochs in all three seeds on HotPotQA and MuSiQue, though on HotPotQA it returned when we trained for four epochs, and matched-reward GRPO barely trains at that rate (\cref{sec:warm-start-results}). PAFT, which learns faster than matched-reward GRPO from prompted inputs, finishes behind it on HotPotQA and MuSiQue (\cref{sec:warm-start-results}). As implemented, $J_Q$ estimators teacher-force one canonical completion, so unlike GRPO they do not benefit from a verifier that accepts other response formats. Annealed-$q$ schedules, including the three-phase recipe below, are untested. The convergence analysis is stylized: the lower bound survives interference between examples and discrete minibatch steps, but assumes exact per-example gradients (\cref{sec:beyond-single-example}), the matching upper bound is for a single example under gradient flow, and both assume a bounded score and exact-match supervision; general rewards are open.

\paragraph{A three-phase post-training recipe.}
The continuum prescribes a regime-dependent recipe: at cold start ($P_\vtheta \approx 0$), GARL at large $q$ ($\geq 0.75$, scaling up with model size) breaches the $\Omega(1/p_0)$ bottleneck (PAFT degenerates here); with prompted inputs, PAFT at intermediate $q$ is the more stable estimator; as $P_\vtheta \to 1$, the bias shrinks and annealing $q \to 0$ recovers the unbiased RB-RLOO estimator. Validating these switches empirically is future work.

\paragraph{Multi-example dynamics.}
Apart from the lower bound (\cref{sec:beyond-single-example}), our convergence analysis considers a single example. Across examples, the dynamics on each $p_i$ involve the kernel $K_{ij} = \nabla_\vtheta P_i \cdot \nabla_\vtheta P_j$. Its interplay with the $q$-dependent weighting $P_j^{-q}$ (potentially via NTK theory) could give matching upper bounds and characterize how dataset-level coverage emerges from gradient-level amplification.

\paragraph{Annealing and richer posterior sampling.}
Principled schedule design adaptive to the current $P_\vtheta$, and automatic switching between GARL and PAFT, remain open. PAFT's importance resampling from the prior pool fails at cold start (vanishing attenuation and particle degeneracy); learned proposals, MCMC, or infilling models conditioned on both $\vx^*$ and $\vy^*$ could extend PAFT to lower-$P_\vtheta$ regimes.

\section*{Broader Impacts}

This work is methodological: we propose a loss family and corresponding gradient estimators for training reasoning language models, using publicly available checkpoints (Qwen 3) and benchmarks (FinQA, HotPotQA, MuSiQue); no new pre-trained models or datasets are released. The $J_Q$ continuum and its estimators (GARL, PAFT) enable post-training without annotated rationales, lowering the data bar for adapting reasoning models to specialized domains, low-resource languages, or settings where rationale annotations are expensive or unavailable. As with any post-training improvement, our methods could in principle be applied to fine-tune models for harmful applications; the same dual-use considerations apply to any RL-based post-training method (\emph{e.g.}, GRPO, RLHF), and our contributions at the level of the training objective remain compatible with existing safety-relevant training procedures.

\clearpage

\end{document}